\documentclass[final,12pt]{clear2025}

\usepackage{tikz}
\usepackage{tikz-cd}
\usetikzlibrary{cd,arrows,calc,shapes,positioning}
\usepackage{bm}
\usepackage{mathtools}
\usepackage[capitalize]{cleveref}

\usepackage{thmtools}
\declaretheorem[name=Theorem, refname={Theorem,Theorems}, Refname={Theorem,Theorems}, numberwithin=section]{theorem}
\declaretheorem[name=Lemma, refname={Lemma,Lemmas}, Refname={Lemma,Lemmas}, sibling=theorem]{lemma}
\declaretheorem[name=Proposition, refname={Proposition,Propositions}, Refname={Proposition,Propositions}, sibling=theorem]{proposition}
\declaretheorem[name=Corollary, refname={Corollary,Corollaries}, Refname={Corollary,Corollaries}, sibling=theorem]{corollary}
\declaretheorem[name=Definition, refname={Definition,Definitions}, Refname={Definition,Definitions}, sibling=theorem]{definition}
\declaretheorem[name=Remark, style=remark, numbered=no]{remark*}
\usepackage{wrapfig}
\usepackage{amsfonts,amssymb,amsmath}
\usepackage{thm-restate}
\usepackage{wasysym}
\usepackage{algorithm}
\usepackage{algpseudocode}

\newcommand{\bR}{\ensuremath \mathbb{R}}

\newcommand{\bN}{\ensuremath \mathbb{N}}

\newcommand{\cA}{\ensuremath \mathcal{A}}
\newcommand{\cB}{\ensuremath \mathcal{B}}

\newcommand{\cE}{\ensuremath \mathcal{E}}
\newcommand{\cF}{\ensuremath \mathcal{F}}
\newcommand{\cG}{\ensuremath \mathcal{G}}
\newcommand{\cH}{\ensuremath \mathcal{H}}
\newcommand{\cI}{\ensuremath \mathcal{I}}

\newcommand{\cP}{\ensuremath \mathcal{P}}

\newcommand{\cU}{\ensuremath \mathcal{U}}
\newcommand{\cX}{\ensuremath \mathcal{X}}

\renewcommand{\S}{\ensuremath \bm{S}}

\newcommand{\equivcls}[1]{%
  #1/\!{\sim}%
}
\newcommand{\equivclsindexed}[2]{%
  #1/\!{\sim_{#2}}%
}

\makeatletter
\def\moverlay{\mathpalette\mov@rlay}
\def\mov@rlay#1#2{\leavevmode\vtop{%
   \baselineskip\z@skip \lineskiplimit-\maxdimen
   \ialign{\hfil$\m@th#1##$\hfil\cr#2\crcr}}}
\newcommand{\charfusion}[3][\mathord]{
    #1{\ifx#1\mathop\vphantom{#2}\fi
        \mathpalette\mov@rlay{#2\cr#3}
      }
    \ifx#1\mathop\expandafter\displaylimits\fi}
\makeatother
\newcommand{\cupdot}{\charfusion[\mathbin]{\cup}{\cdot}}

\let\wasysymLightning\lightning
\newcommand{\contradiction}{\wasysymLightning}
\newcommand\restr[2]{{%
  \left.\kern-\nulldelimiterspace
  #1
  \vphantom{\big|}
  \right|_{#2}
  }}

\newcommand{\dif}{\mathrm{d}}   % for use in differentials 
\DeclareMathOperator{\bE}{\mathbb{E}} % Expectation of random variable
\newcommand{\dblbrackets}[1]{\left[\!\left[ #1 \right]\!\right]}

\DeclareMathOperator{\indep}{\perp\!\!\!\perp} % independence of random variables
\DeclareMathOperator{\notindep}{\not\! \perp\!\!\!\perp} % not independence of RVs
\newcommand{\given}{\mid} % add conditioning/separation set
\newcommand{\cond}{\: : \:} % add conditioning/separation set

\newcommand{\indepsh}{\indep_{s,h}^+}

\newcommand{\adj}{\hspace{-0.11em}\mathrel{\rule[0.5ex]{1.1em}{0.6pt}}\hspace{-0.1em}}

\newcommand{\rcirclearrow}{\hspace{-0.12em} \leftarrow \hspace{-0.15cm}\circ\hspace{0.06em}} % <----o
\newcommand{\lcirclearrow}{\hspace{+0.06em}\circ\hspace{-0.16cm}\rightarrow \hspace{-0.12em}} % o---->
\newcommand{\lrcirclearrow}{\hspace{+0.06em} \circ\hspace{-0.12cm}\adj\hspace{-0.10cm}\circ\hspace{0.10em}} % o---o

\newcommand{\pa}{\ensuremath{\mathrm{pa}}}

\newcommand{\an}{\ensuremath{\mathrm{an}}}

\newcommand{\scc}{\ensuremath{\mathrm{scc}}}
\newcommand{\chG}[1]{\ensuremath{\mathrm{ch}^{\cG}_{#1}}}
\newcommand{\paG}[1]{\ensuremath{\mathrm{pa}^{\cG}_{#1}}}
\newcommand{\deG}[1]{\ensuremath \mathrm{de}^{\cG}_{#1}}
\newcommand{\ndG}[1]{\ensuremath \mathrm{nd}^{\cG}_{#1}}
\newcommand{\anG}[1]{\ensuremath \mathrm{an}^{\cG}_{#1}}

\newcommand{\sibG}[1]{\ensuremath \mathrm{sib}^{\cG}_{#1}}

\newcommand{\sccG}[1]{\ensuremath \mathrm{scc}^{\cG}_{#1}}

\newcommand{\coll}[1]{\ensuremath \mathrm{coll}_{#1}}
\newcommand{\ncoll}[1]{\ensuremath \mathrm{ncoll}_{#1}}
\newcommand{\ncollub}[1]{\ensuremath \mathrm{ncoll}_{ub, #1}}
\newcommand{\ncollb}[1]{\ensuremath \mathrm{ncoll}_{b, #1}}

\tikzstyle{obs} = [circle,fill=white,draw=black,inner sep=1pt,minimum size=25pt,node distance=0.75cm,thick]%font=\fontsize{10}{10}\selectfont,
\tikzstyle{latent} = [obs,dotted]
\tikzstyle{fixed} = [obs,rectangle,minimum size=22pt]
\tikzstyle{small} = [inner sep=0pt,minimum size=20pt]

\definecolor{Grey}{rgb}{0.5,0.5,0.5}
\definecolor{NavyBlue}{rgb}{0.0, 0.0, 0.5}
\definecolor{ForestGreen}{rgb}{0.13, 0.55, 0.13} 
\newcommand{\defined}[1]{{\color{NavyBlue}\emph{#1}}}
\newcommand{\defineddeutsch}[1]{{\color{Grey}\emph{#1}}}
\newcommand{\symbdef}[1]{{\color{ForestGreen}#1}}
\newcommand{\xhdr}[1]{\noindent \textbf{#1}\:}

\title[Causal Discovery on Path Spaces]{An Asymmetric Independence Model\\ for Causal Discovery on Path Spaces}

\usepackage{times}

\clearauthor{%
 \Name{Georg Manten} \Email{georg.manten@tum.de}\\
 \addr{Technical University of Munich, Germany\\
       Helmholtz Munich, Germany \\
       Munich Center for Machine Learning (MCML), Germany}
 \AND
 \Name{Cecilia Casolo} \Email{cecilia.casolo@tum.de}\\
 \addr{Technical University of Munich, Germany\\
       Helmholtz Munich, Germany \\
       Munich Center for Machine Learning (MCML), Germany}
 \AND
 \Name{Søren Wengel {Mogensen}}\footnotemark[2] \Email{swm.fi@cbs.dk}\\
 \addr{Department of Finance, Copenhagen Business School, Denmark}
 \AND
 \Name{Niki Kilbertus}\footnotemark[2] \Email{niki.kilbertus@tum.de}\\
 \addr{Technical University of Munich, Germany\\
       Helmholtz Munich, Germany \\
       Munich Center for Machine Learning (MCML), Germany}
\footnotetext[2]{These authors contributed equally.}
}

\begin{document}

\maketitle

\begin{abstract}%
We develop the theory linking `E-separation' in directed mixed graphs (DMGs) with conditional independence relations among coordinate processes in stochastic differential equations (SDEs), where causal relationships are determined by ``which variables enter the governing equation of which other variables''. We prove a global Markov property for cyclic SDEs, which naturally extends to partially observed cyclic SDEs, because our asymmetric independence model is closed under marginalization. We then characterize the class of graphs that encode the same set of independence relations, yielding a result analogous to the seminal `same skeleton and v-structures' result for directed acyclic graphs (DAGs). In the fully observed case, we show that each such equivalence class of graphs has a greatest element as a parsimonious representation and develop algorithms to identify this greatest element from data. We conjecture that a greatest element also exists under partial observations, which we verify computationally for graphs with up to four nodes. 
\end{abstract}


\section{Introduction and Related Work}\label{sec:intro}

Discovering causal relationships from observational data holds great promise across many domains: in biology one may aim at inferring which genes regulate which other genes from single cell transcriptomics; in medical settings, one strives to understand the causal interactions between different diseases and symptoms; in finance, successful trading strategies require an understanding of causal drivers; in complex engineered systems, uncovering the underlying causal processes is crucial for effective predictive maintenance.
Causal relationships between different variables are often described by graphs.
The predominant graphical framework for causal discovery (or causal structure learning) are (variants of) Structural Causal Models (SCM) \citep{pearl2009causality,peters2017elements}, which can be considered a description of the underlying data generating mechanism.
SCMs encode causal assumptions in the form of Directed Acyclic Graphs (DAGs) and leverage information encoded in the DAG about the observed distribution in terms of conditional independencies for causal inferences \citep{pearl1995causal,lauritzen2001causal}.
Learning causal structures from observational data within this framework is a high-impact endeavor that has received considerable attention over the recent decades \citep{spirtes2000causation,spirtes1995causal, zheng2018dags,chickering2002optimal,vowels2022d}.

Most existing works consider `static models',  where independence relations are symmetric.
In this work, we focus on dynamic systems that evolve in continuous time.
Thus, instead of assuming a static joint observational distribution, we consider as observations multi-variate stochastic processes.
For these systems, exploiting the direction of time, i.e., the fact that causal influences can never go from the future to the past, can provide additional valuable information.
Leveraging this information requires an asymmetric independence notion that captures the fundamental asymmetry between past and future.

One natural candidate to model causal relations among multivariate stochastic processes is the notion of local independence, which evaluates the dependence of the present on the past among the different coordinate processes.
Local independence captures arbitrarily `fast' interactions between coordinate processes \citep{schweder1970composable,mogensen2018causal}.
Various theoretically oriented works have proposed (conditional) local independence to infer (partial) causal graphs in, e.g., point process models \citep{didelez2008graphical,meek2014toward,mogensen2018causal,mogensen2020markov}.
Unfortunately, as of now local independence cannot be tested in practice with only one exception:
\citet{CPH23} propose a practical test specifically for counting processes.
There exists no statistical test of local independence for other classes of stochastic processes such as diffusions.

Another common approach to causal discovery in dynamical settings is to extend the methodology from the static case by assuming stochastic processes to be observed on a fixed regular time grid.
One can then `unroll' the causal graph in time and assume a discrete (auto-regressive) law governing the observed data \citep{runge2018causal,runge2019detecting,runge2020discovering,runge2023causal}.
The `discrete regular time grid' assumption fundamentally limits these approaches: \textbf{(i)} They require observations of all coordinate processes at fixed, evenly spaced time points, disallowing for irregular sampling or missing observations. \textbf{(ii)} All inferences critically depend on the sampling frequency. \textbf{(iii)} Methods are typically sensitive to the unknown maximal interaction time (i.e., the lag), which has to be chosen heuristically.
Practically, these approaches also suffer from requiring large numbers of (symmetric) conditional independence tests for causal discovery \citep{shah2020hardness,lundborg2022conditional}.

Recently, continuous time systems have received attention in approaches that examine dependence on entire path-segments.
\citet{laumann2023kernel} developed a test for conditional independence tailored to functional data, including path-valued random variables. Their approach treats these variables as standard random variables, ignoring the temporal nature, and utilizes established results from static structural causal models (SCMs). They apply the resulting global Markov property for acyclic graphs to perform causal discovery using traditional algorithms (e.g., the PC algorithm \citep{spirtes2000causation}) by testing functional data on the full intervals.
Instead, \citet{manten2024sigker} recently introduced both a conditional independence test and a Markov property that explicitly accounts for time.
By partitioning the time interval into a past and a future segment, this approach enables more informative causal discovery.
However, their approach is still restricted to \emph{acyclic} SDE models.
\citet{boekenmooij2024dynamicstructuralcausalmodels} apply the Markov property found by \citet{mooij2023constraintbased} to path-valued random variables using \emph{symmetric} 
$\sigma$-separation to model cyclic dependencies.
While allowing for cyclic dependencies -- arguably one of the key characteristics of dynamical systems -- this approach again ignores the direction of time, leading to weaker results than can be achieved by exploiting temporal order.

In this paper, we develop a graphical framework that offers several contributions: \textbf{(i)} it allows for cycles (unlike \citealp{manten2024sigker,laumann2023kernel}); \textbf{(ii)} it leverages the direction of time (unlike \citealp{boekenmooij2024dynamicstructuralcausalmodels,laumann2023kernel}); \textbf{(iii)} it is practically testable (unlike local independence methods \citealp{mogensen2018causal}); and \textbf{(iv)} it can handle partial observations, i.e., still provides (partial) results in the presence of unobserved confounding.
The last point necessitates a constraint-based approach, which we focus on in this work.
We achieve \textbf{(i)-(iv)} by a novel, more informative asymmetric version of $\sigma$-separation, called E-separation. Our criterion can be practically tested in partially observed cyclic SDE models using the conditional independence tests developed by \citet{manten2024sigker}. For empirical validation of our theoretical contributions, we provide two example experiments in \cref{app:causal_discovery_algorithm}.
In summary, our framework is strictly more comprehensive than existing methods to constraint-based causal discovery in continuous time dynamical systems along criteria \textbf{(i)-(iv)}.

\section{Stochastic Differential Equations, Graphs, and Separation}\label{sec:background}

\xhdr{Data generating process.}
We focus on processes arising from stochastic differential equations (SDEs). SDEs are often used to model a variety of systems in physics, health, finance, and beyond, and also allow for a natural causal interpretation between the different coordinate processes.
Following \citet{manten2024sigker}, we assume a stationary, path-dependent `SDE model'
\begin{equation}\label{eq:SDEmodel}
    \begin{cases}
        \dif X^k_t = \mu^k(X_{[0,t]})\dif t + \sigma^k(X_{[0,t]})\dif W^k_t, \\
        X^k_0 = x^k_0 \qquad \text{for } k \in [d] := \{1,\ldots,d\}\:,
    \end{cases}
\end{equation}
with potentially multi-dimensional $X_t^k \in \bR^{n_k}$, $n_k \in \bN_{>0}$, Brownian motions $W_t^k \in \bR^{m_k}$, $m_k \in \bN_{>0}$ and drift $\mu^k : C^{0}([0,1],\bR^{n}) \rightarrow \bR^{n_k}$ and diffusion $\mu^k : \cX^\prime \rightarrow \bR^{n_k \times m_k}$ being functions of solution paths of the SDE up to time $t$. We assume the initial conditions $x^k_0$ and Brownian motions are jointly independent. 
The structure of \Cref{eq:SDEmodel}, where each coordinate process influences the differential of others, naturally defines a directed ``dependency graph'' $\cG = (V, \cE_{d})$ with vertices $V = [d] := \{1, \ldots, d\}$ representing the individual coordinate processes. In this graph,
$i$ is a parent of $j$ ($i \in \pa^{\cG}_j$) when either $\mu^j$ or $\sigma^j$ is not constant in the $i$-th argument. Importantly, unlike \citet{manten2024sigker}, we do not impose acyclicity on this graph and allow for partial observations.

\xhdr{Basic graph terminology.}
We follow common notation and terminology from \citet{mogensen2020markov,forre2023causality,peters2017elements} and provide a concise summary of the required concepts here.
\begin{definition}[Directed (Mixed) Graph, D(M)G]
A \defined{directed mixed graph (DMG)} is a triple $\cG = (V, \cE_{d}, \cE_{bi})$ with a set of \textit{nodes} $V\cong [d]$, $d \in \bN$ and two sets of edges, the \defined{directed edges} $\cE_{d} \subset V \times V =: \cE_{d}(V)$ and the \defined{bidirected edges} $\cE_{bi} \subset \equivcls{\left((V \times V) \right)}=: \cE_{bi}(V)$ with equivalence relation $(v_1, v_2) \sim (v_3, v_4) : \Leftrightarrow \left( (v_1 = v_3) \land (v_2 = v_4) \right) \lor \left( (v_1 = v_4) \land (v_2 = v_3) \right) $ and equivalence classes denoted by \symbdef{$[(j,k)]_{bi}$}. $\cG$ is called \defined{directed graph (DG)}, if $\cE_{bi} = \emptyset$.
The \defined{induced subgraph} on $A \subset V$ by $\cG$ is the DMG $\symbdef{\restr{\cG}{A}} = \left( A , \cE_{d} \cap \cE_{d}(A),\cE_{bi} \cap \cE_{bi}(A) \right)$.
\end{definition}
Throughout this paper, $\cG$ will denote a DG or DMG, $V \cong [d]$ its node set (corresponding to coordinate processes) and $\cE_d$, $\cE_{bi}$ its edges.
For simplicity, we sometimes write $\symbdef{v \in \cG}$ for $v \in V$, $\symbdef{v_1 \to v_2 \in \cG}$ for $(v_1, v_2) \in \cE_{d}$ (we then say that the edge has a \defined{tail} at $v_1$ and a \defined{head} at $v_2$, also referred to as \defined{edge marks}), and
$\symbdef{v_1 \leftrightarrow v_2 \in \cG}$ for $[(v_1,v_2)]_{bi} \in \cE_{bi}$.
Moreover, $\symbdef{v_1 \lcirclearrow v_2 \in \cG}$ symbolizes either $v_1 \to v_2 \in \cG$ or $v_1 \leftrightarrow v_2 \in \cG$
and $\symbdef{v_1 \sim v_2 \in \cG}$ or $\symbdef{v_1 \lrcirclearrow v_2 \in \cG}$ that $v_1 \to v_2 \in \cG$ or $v_1 \leftarrow v_2 \in \cG$ or $v_1 \leftrightarrow v_2 \in \cG$, meaning that $v_1$ and $v_2$ are \defined{adjacent} in $\cG$.
The circle $\circ$ is like a placeholder for either an `arrowhead' or a `tail'.
For nodes $v,w \in \cG$, a \defined{walk from $v$ to $w$ in $\cG$}, is a finite sequence $\{ (v_k, e_k) \}_{k \in [n]}$ such that $v_0 := v$, $v_{n+1}:= w$ and $e_k \in \{ (v_k, v_{k+1}), (v_{k+1}, v_{k}),[v_k, v_{k+1}]_{bi} \}$ $\forall k \in [n]$ and often denoted $v = v_0 \overset{e_0}{\sim} v_1 \overset{e_1}{\sim} \ldots\overset{e_n}{\sim} v_{n+1} = w.$
We denote by $\pi^{-1}$ the \defined{inverse walk}, $w = v_{n+1} \overset{e_n}{\sim} v_n \overset{e_{n-1}}{\sim} \ldots \overset{e_0}{\sim} v_{0} = v$.
The \defined{trivial walk} (from $v$ to $v$), $\emptyset$ is the walk consisting of only a single node $v \in \cG$ and no edges.
We call the walk $\pi: v \sim \ldots \sim w$ \defined{bidirected} (\defined{directed}) if $e_k = [ v_k, v_{k+1} ]_{bi}\in \cE_{bi}$ ($e_k = (v_k, v_{k+1}) \in \cE_d$) for all $k \in [n]$.  
A walk $v=v_{0} \sim \ldots \sim v_{n+1}=w$ is called \defined{path from $v$ to $w$ in $\cG$}, if $\lvert \{ v_0, \ldots , v_n \} \rvert = n+1$, i.e., no node besides the endpoint occurs more than once.
Definitions of (bi)directedness for walks carry over directly from the corresponding definitions for paths.

Walks $\pi_1 = v \overset{e_0^1}{\sim} v^1_1 \overset{e_1^1}{\sim} \ldots \overset{e_n^1}{\sim} w$ and $\pi_2 = v \overset{e_0^2}{\sim} v^2_1 \overset{e_1^2}{\sim} \ldots \overset{e_m^2}{\sim} w$ are \defined{endpoint-identical}, if the marks of $e_0^1, e_0^2$ at $v$ and of $e_n^2, e_m^2$ at $w$ agree.
A DMG $\cG = (V, \cE_{d}, \cE_{bi})$ is \defined{acyclic}, if there exists no non-trivial directed walk $v \rightarrow \ldots \rightarrow v$ for all $v \in V$ $\cG$, making it an \defined{acyclic directed mixed graph (ADMG)} and an \defined{directed acyclic graph (DAG)} if in addition $\cE_{bi} = \emptyset$.

Let $\cG_1 = (V, \cE^1_{d}, \cE^1_{bi})$, $\cG_2 = (V, \cE^2_{d}, \cE^2_{bi})$ be DMGs.
We call $\cG_1$ \defined{subgraph} of $\cG_2$ ($\cG_2$ \defined{supergraph} of $\cG_1$) and write $\cG_1 \subseteq \cG_2$, if $\cE^1_d \subseteq \cE^2_d$ and $\cE^1_{bi} \subseteq \cE^2_{bi}$. Note: $\subseteq$ is a partial order over the set of DMGs over nodes $V$. Furthermore, we define the \defined{complete DG (DMG)} to be the graph with all possible edges, $(V, \cE_d = V \times V)$ ($(V, \cE_d = V \times V, \cE_{bi}=\equivcls{(V \times V)})$).
Finally, for a node $v$ in a DMG $\cG$ we define its
\defined{parents} $\paG{v} := \{w \in \cG \mid w \to v \in \cG\}$,
\defined{children} $\chG{v} := \{w \in \cG \mid v \to w \in \cG\}$,
\defined{siblings} $\sibG{v} := \{w \in \cG \mid w \leftrightarrow v \in \cG\}$,
\defined{ancestors} $\anG{v} := \{w \in \cG \mid \exists \: \text{directed path} \: w \to \ldots \to v \in \cG \}$,
\defined{descendants} $\deG{v} := \{w \in \cG \mid \exists \: \text{directed path} \: v \to \ldots \to w \in \cG\}$,
and \defined{non-descendants} $\ndG{v} := V \setminus \deG{v}$.
All notions extend to sets of nodes by unions, e.g., $\paG{A} := \bigcup_{v\in A} \paG{v}$, with $\ndG{A} := V \setminus \deG{A}$ for $A \subseteq V$.
Note, that $v \in \anG{v}, A \subseteq \anG{A}, v \in \deG{v}, A \subseteq \deG{A}$ because we allow for trivial paths.

When dealing with cycles, we also need the notion of the \defined{strongly connected component of $v$ in $\cG$}, denoted by $\sccG{v} := (\anG{v} \cap \deG{v})$, having $v \in \sccG{v}$.
The \defined{set of strongly connected components} of $\cG$ is indicated by $\mathbf{S} (\cG) := \{ A \subseteq V \mid \exists j \in V : A = \sccG{j} \}$ and it also defines an equivalence relation on $V$ given by $v \sim_{1} w : \Leftrightarrow w \in \sccG{v}$ such that the equivalence classes partition the node set $V$.
The \defined{DAG of strongly connected components} for a DG $\cG=(V,\cE_d)$ is denoted by $\mathbf{S} (\cG) := (\equivclsindexed{V}{1}, \equivclsindexed{(\cE_{d}\setminus \Delta_V)}{2})$ with the equivalence relation $(v \rightarrow w) \sim_{2} (v^\prime, w^\prime) :\Leftrightarrow (v \sim_1 v^\prime) \land (w \sim_1 w^\prime )$.

\xhdr{$\sigma$-/$d$-Separation.}
To make this paper self-contained, we proceed by briefly recalling the relevant existing separation notions.
Given a DMG $\cG = (V, \cE_{d}, \cE_{bi})$, we call a node $v_k$ (or rather the position $k \in \{ 0, \ldots , n+1 \}$) on a walk $\pi$ a \defined{non-collider on $\pi$}, if it is an end-point ($k \in \{0,n+1\}$), in a left-chain ($v_{k-1} \leftarrow v_{k} \rcirclearrow v_{k+1}$), in a right-chain ($v_{k-1} \lcirclearrow v_{k} \rightarrow v_{k+1}$), or in a fork ($v_{k-1} \leftarrow v_k \rightarrow v_{k+1}$).
It is a \defined{collider on $\pi$}, if $v_{k-1} \lcirclearrow v_k \rcirclearrow v_{k+1}$.
The set of colliders/non-colliders of $\pi$ is denoted $\symbdef{\coll{\pi}}$/$\symbdef{\ncoll{\pi}}$.
We further call a non-collider $v_k$ on walk $\pi$ an \defined{unblockable non-collider on $\pi$} if $k \notin \{ 0,n \}$ and
it is in a left-chain ($v_{k-1} \leftarrow v_{k} \rcirclearrow v_{k+1} \: \land \: v_{k-1} \in \sccG{v_k}$),
it is in a right-chain ($v_{k-1} \lcirclearrow v_{k} \rightarrow v_{k+1}  \: \land \: v_{k+1} \in \sccG{v_k}$),
or it is in a fork ($v_{k-1} \leftarrow v_k \rightarrow v_{k+1}  \: \land \: (v_{k-1} \in \sccG{v_k} \: \land \: v_{k+1} \in \sccG{v_k})$).
Otherwise, we call $v_k$ a \defined{blockable non-collider on $\pi$}. Similarly, we denote the set of blockable/unblockable non-colliders of $\pi$ by $\symbdef{\ncollb{\pi}}$/$\symbdef{\ncollub{\pi}}$.

With this setup, we can define \emph{d-separation} for acyclic graphs (e.g., \citep{pearl_d_sep} for DAGs).
\emph{d-separation:} A walk $\pi =  v_0 \overset{e_0}{\sim} \ldots \overset{e_n}{\sim} v_{n+1}$ is called \defined{$d$-open given $C$ or $d$-$C$-open} for $C \subseteq V$, if $\coll{\pi} \subseteq \anG{C}$ and $\ncoll{\pi} \cap C = \emptyset$. Otherwise it is called \defined{$d$-blocked given $C$ or $d$-$C$-blocked}, meaning that $\coll{\pi} \not\subseteq \anG{C}$ or $\ncoll{\pi} \cap C \neq \emptyset$.
If each walk $\pi = a \sim \ldots \sim b$ between sets $A,B \subseteq V$ is $d$-$C$-blocked, we call \defined{$A$ $d$-separated from $B$ given $C$}, (\symbdef{$A \indep_{d}^\cG B \given C$}); otherwise we write $\symbdef{A \notindep_{d}^\cG B \given C}$.

A more suitable extension to cyclic settings is the notion of \textit{$\sigma$-separation}, introduced by \citet{forre2017markovpropertiesgraphicalmodels}.
\emph{$\sigma$-separation:} A walk $\pi =  v_0 \overset{e_0}{\sim} \ldots \overset{e_n}{\sim} v_{n+1}$ is called \defined{$\sigma$-open given $C$ or $\sigma$-$C$-open} for $C \subseteq V$, if $\coll{\pi} \subseteq \anG{C}$ and $\ncollb{\pi} \cap C = \emptyset$.
Otherwise it is called \defined{$\sigma$-blocked given $C$ or $\sigma$-$C$-blocked}, meaning that $\coll{\pi} \not\subseteq \anG{C}$ or $\ncollb{\pi} \cap C \neq \emptyset$.
We analogously extend this notion to sets to define
\defined{$A$ $\sigma$-separated from $B$ given $C$}, (\symbdef{$A \indep_{\sigma}^\cG B \given C$}) and otherwise write $\symbdef{A \notindep_{\sigma}^\cG B \given C}$.

\section{A Dynamic Global Markov Property for DMGs on Path-Space}

\subsection{Lifted Dependency Graph and Separation}

We now define an extension of the \emph{lifted dependency graph} introduced in \citet{manten2024sigker}, which enables us to leverage the direction of time.

\xhdr{Lifted dependency graph.}
For a DMG $\cG = (V, \cE_d, \cE_{bi})$, the \defined{lifted dependency graph} $\tilde{\cG}$ is the DMG $\tilde{\cG} = (\tilde{V}, \tilde{\cE}_d, \tilde{\cE}_{bi})$, where $\tilde{V} := V_0 \sqcup V_1$ is a disjoint union of two copies of $V$, with nodes subscripted by $0/1$ to indicate set membership and edges $\tilde{\cE}_{d}:= \{ (u_0,v_0), (u_0,v_1), (u_1,v_1) \cond (u,v) \in \cE_{d} \}$ and $\tilde{\cE}_{bi}:= \{ [(u_0,v_0)], [(u_0,v_1)], [(u_1, v_0)], [(u_1,v_1)] \cond [(u,v)] \in \cE_{bi} \}$.

\begin{wrapfigure}{r}{0.45\textwidth}
    \centering
    \vspace{-5mm}
    \includegraphics[width=0.45\textwidth]{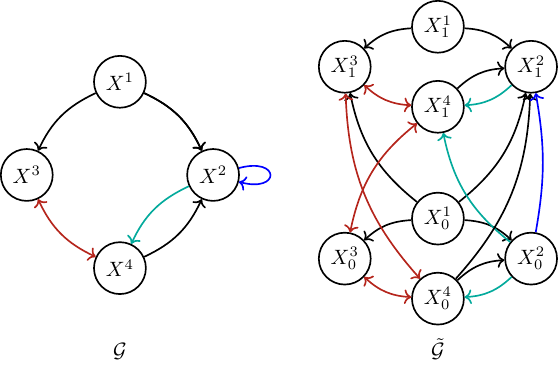}
    \vspace{-5mm}
    \caption{The lifted dependency graph $\tilde{\cG}$ for the DMG $\cG$.}
    \label{fig:lifted_dependence_graph_with_scc}
    \vspace{-8mm}
\end{wrapfigure}

Next we introduce a graphical separation criterion on the original DMG $\cG$ in terms of $\sigma$-separation and its lifted dependency graph.

\xhdr{E-Separation.}
For a DMG $\cG = (V , \cE_{d},\cE_{bi})$, $B$ is \defined{E-separated} from $A$ given $C$ in $\cG$ if
\begin{align*}
    A_0 \indep_{\sigma}^{\tilde{\cG}} B_1 \given C_0, C_1\setminus B_1,
\end{align*}
where $A,B,C \subset V$, $A_0 = \{ a_0 : a \in A \} \subseteq \tilde{V}$ and $B_1 = \{ b_1 : b \in B \} \subseteq \tilde{V}$, $C_0:= \{ c_0 : c \in C \} \subseteq \tilde{V}$, $C_1:= \{ c_1 : c \in C \} \subseteq \tilde{V}$ in the lifted dependency graph $\tilde{\cG}$ of $\cG$.
We denote this E-separation by $\symbdef{A \indep_{E}^{\cG} B \given C}$.
An example DMG $\cG$ and its lifted dependency graph are shown in \cref{fig:lifted_dependence_graph_with_scc}, illustrating how time-ordering can be modeled graphically.
The asymmetric notion implies that $X^{3}$ cannot be E-separated from $X^{1}$ due to the edge $X_0^1 \rightarrow X_0^3$, however, $X^{1}$ can be E-separated from $X^{3}$ as each walk is blocked by the empty set.
Moreover $X^2$ and $X^4$ can not be E-separated from $X^1$ due to the edge $X_{0}^{1} \rightarrow X_{1}^{2}$ directly feeding into the future strongly connected component $\{ X_1^2, X_1^4 \}$. Therefore, $X^{1}$ and $X^{4}$ are not separable, even though there is no directed edge $X^{1} \rightarrow X^{4}$.

\subsection{Asymmetric Graphoids}

Analogous to properties found in the literature on conditional independence models (see, e.g., \citealp{lauritzen1990independence}), we can also define \textit{abstract independence models} over a finite set $V$, generically denoted here by $\cI \subseteq \cP(V)^3$, which are sets of triples $(A,B,C)\subseteq V$ that admit for certain properties and encode important relations. For example, if $(A,B,C)$ is inside $\cI$, a set $A$ is `independent' from set $B$ given set $C$ (or `separated' when referring to graphical criteria).
Unlike traditional frameworks such as the notion of $\sigma$-separation, which is symmetric in $A$ and $B$, the unidirectionality of time in our context necessitates an asymmetric independence model $\cI \subseteq \cP (V)^{3}$, whose properties are sometimes referred to as \defined{asymmetric (semi) graphoid properties} (see \citet{didelez2008graphical,mogensen2020markov}):
\vspace{-0.2cm}
\begin{itemize}
        \setlength{\itemsep}{-2pt}
        \item[(LR)]\defined{Left redundancy}: $(A,B,A) \in \cI$
        \item[(RR)]\defined{Right redundancy}:  $(A,B,B) \in \cI$
        \item[(LD)]\defined{Left decomposition}: $(A,B,C) \in \cI$, $D \subset A$ $\Rightarrow$ $ (D,B,C) \in \cI$
        \item[(RD)]\defined{Right decomposition}: $(A,B,C)\in \cI$, $D \subset B$ $\Rightarrow$ $(A,D,C) \in \cI$
        \item[(LWU)]\defined{Left weak union}: $(A,B,C) \in \cI$, $D \subset A$ $\Rightarrow$ $(A,B, C \cup D)  \in \cI$
        \item[(RWU)]\defined{Right weak union}: $(A,B,C) \in \cI$, $D \subset B$ $\Rightarrow$ $(A,B, C \cup D) \in \cI$
        \item[(LC)]\defined{Left contraction}: $(A,B,C) \in \cI$, $(D,B, A \cup C) \in \cI$ $\Rightarrow$ $(A \cup D, B,C) \in \cI$
        \item[(RC)]\defined{Right contraction}: $(A,B,C) \in \cI$, $(A,D,B \cup C) \in \cI$ $\Rightarrow$ $(A, B \cup D, C) \in \cI$
        \item[(LI)]\defined{Left intersection}: $(A,B,C) \in \cI$, $(C,B,A) \in \cI$ $\Rightarrow$ $(A \cup C, B, A \cap C) \in \cI$
        \item[(RR)]\defined{Right intersection}:  $(A,B,C) \in \cI$, $(A,C,B) \in \cI$ $\Rightarrow$ $(A, B \cup C, B \cap C) \in \cI$.
        \item[(LCo)]\defined{Left composition}: $(A,B,C) \in \cI$ $\Leftrightarrow$ $(\{ a \}, B, C)\in \cI$ $\forall \: a \in A$
        \item[(RCo)]\defined{Right composition}:  $(A,B,C) \in \cI$ $\Leftrightarrow$ $ (A, \{ b \}, C) \in \cI$ $\forall \: b \in B$
\end{itemize}

\begin{proposition}[Ternary relation defined by $\indep_{E}^{\cG}$]\label{prop:esep_ternary}
Let $\cG = (V, \cE_d)$ be a DG. Then $\indep_{E}^{\cG}$ defines the following ternary relation on the node set $V$,
\begin{align*}
    \cI_{E} := \cI_{E}^{\cG} := \lbrace  (A,B,C) \in \cP (V)^3 \cond A \indep_{E}^{\cG} B \given C \rbrace
\end{align*}
which satisfies (LR), (LD), (RD), (LC), (LCo), (RCo).
\end{proposition}
We defer all proofs to \cref{app:proofs} due to space restrictions.
The proof of \cref{prop:esep_ternary} is in \cref{app:ternary_proofs}.
Note that right redundancy (RR) does not hold, e.g., consider the graph $\cG = (V = \{ a,b \}, \cE_{d} = \{ (a,b) \})$.
In addition, (LWU) does not hold with $\cG = (V = \{ a,b,d,e \}, \cE_d = \{ (a,e), (b,e), (e,d) \})$ yielding a counterexample for $A = \{ a \}$, $B = \{ b \}$, $D = \{ d \}$ and $C = \emptyset$.
For causal discovery, the notion of `separability' is important, which we define for abstract independence models.
\begin{definition}[Separability]
Let $\cI \subseteq \cP(V)^3$ be an independence model over $V$ and $a,b \in V$. We call $b$ \defined{separable} from $a$ if there exists a $C \subseteq V \setminus \{ a \}$ such that $(\{ a \},\{ b \},C) \in \cI$ and otherwise
\defined{inseparable}.
We denote the set of nodes inseparable from a node $b$ by 
\begin{align*}
    u (b, \cI_{E}^{\cG}) := \left\lbrace a \in V \cond (\lbrace a \rbrace,\lbrace b \rbrace,C) \notin \cI_{E}^{\cG}, \: C \subseteq V \setminus \lbrace a \rbrace \right\rbrace \:.
\end{align*}
\end{definition}
We adapt a result by \citet[Prop.~3.5]{mogensen2020markov} to relate paths and walks.
\begin{lemma}\label{lemma:lemma_walk_to_path}
Let $\cG = (V, \cE_d, \cE_{bi})$ be a DMG, $\tilde{G}$ its lifted dependency graph, $a,b \in V$ and $C \subseteq V\setminus \lbrace a \rbrace$. Then it holds: If there exists a $C_0 \sqcup C_1 \setminus \lbrace b_1 \rbrace$-$\sigma$-open walk $\pi: a_0 \sim \ldots \sim b_1$ in $\tilde{G}$, there exists $C_0 \sqcup C_1 \setminus \lbrace b_1 \rbrace$-$\sigma$-open path $\tilde{\pi} : a_0 \sim \ldots \sim b_1$ in $\tilde{\cG}$ consisting of edges from $\pi$. 
\end{lemma}
The proof can be found in \cref{app:proofs_on_graphs}.

\subsection{Conditional Independence}

In this section, we introduce our independence concept for stochastic dynamical systems based on an asymmetric notion of conditional independence that also takes into account time.
We then show that it satisfies the global Markov property with respect to a DG using E-separation.
\begin{definition}[Future-extended $h$-locally CI]\label{def:shCI}
Let $\{ X^{i}_t \}_{i \in [d]}$ be the coordinate processes of a solution of \cref{eq:SDEmodel} for $t \in [0,1]$, $[0,s], [s,s+h] \subseteq [0,1]$ be subintervals for $s,h > 0$ and $A,B,C \subseteq [d]$.
We say that $X^A$ is \defined{future-extended $h$-locally conditionally independent (CI)} of $X^B$ given $X^C$ at $s$, (\symbdef{$X^{A} \indepsh X^{B} \given X^C$}) if
\begin{align}\label{eq:general_future_h}
X^A_{[0,s]} \indep X^B_{[s,s+h]} \given X^{C}_{[0,s]}, X^{C \setminus B}_{[0,s+h]}
\end{align}
\end{definition}
\begin{remark*}
By the CI statement \eqref{eq:general_future_h}, we refer to the independence of increments rather than the independence of consecutive path segments. By construction, path-valued random variables $\omega \mapsto ([a,b] \ni t \mapsto X^i_t(\omega) - X^i_a(\omega))$, $i \in [d]$,  do not depend on the initial conditions and only on subsequent increments. For brevity, we denote these paths by $X_t^{i}$. 
We note that signature kernel–based CI-tests implicitly address this nuance (see, e.g., \citet{manten2024sigker}) as the signature transform is translation invariant ($S(X(t))_{a,b} = S(X(t)- X(a))_{a,b}$).
\end{remark*}
\Cref{def:shCI} differs from the one by \citet{manten2024sigker} in that $B$ is not necessarily inside the conditioning set $C$.
We now embed $\indepsh$ into the context of asymmetric independence models.
\begin{proposition}[$\indepsh$ as a ternary relation]\label{prop:indepsh_ternary}
    Under the conditions of \Cref{def:shCI}, $\indepsh$ defines the following ternary relation on $V$, 
    \begin{align*}
        \cI_{s,h} = \lbrace (A,B,C) \in \cP(V)^3 \cond X^{A} \indepsh X^{B} \given X^{C} \rbrace
    \end{align*}
    which satisfies (LR) and (LD).
\end{proposition}
Because \Cref{prop:esep_ternary,prop:indepsh_ternary} show that both independence models satisfy (LR) and (LD), we assume from now on that $A \cap C = \emptyset$ in statements of the form $(A,B,C) \in \cI_E$ ($\cI_{s,h}$ respectively). 

\begin{proposition}[Global Markov property for E-separation and $\indepsh$]\label{prop:global_MP_emilio_granger}
 Let $\{ X^{i}_t \}_{i \in [d]}$ be the coordinate processes of a solution of \cref{eq:SDEmodel} for $t \in [0,1]$, $\cG = (V \cong [d], \cE_{d})$ the adjacency graph defined by \cref{eq:SDEmodel}, $A,B,C \subseteq V$.
 Then
    \begin{align}\label{eq:sigma_global_MP}
        A \indep_{E}^{\cG} B \given C \: \Rightarrow \: X^{A} \indepsh X^{B} \given X^{C}.
    \end{align}
\end{proposition}
The proof relies on an application of \citet[Theorem B.2]{Forre_Constraint} based on the concept of \textit{acyclification} and is worked out in detail in \cref{app:proofmarkov}.

If the summary graph $\cG = (V, \cE_d)$ is acyclic except for self-loops, the lifted dependency graph $\tilde{\cG}$ becomes a DAG, hence all non-colliders are blockable and E-separation reduces to $A_0 \indep_{d}^{\tilde{\cG}} B_1 \given C_0, C_1\setminus B_1$.
Hence, the Markov property found by \citet{manten2024sigker} is a simple special case of ours.
This special case can be used in causal discovery algorithms that identify the full graph in the acyclic case (including self-loops) by testing $\indepsh$ ( under a faithfulness assumption). 
However, when allowing for cycles beyond self-loops, i.e., there are strongly connected components of size at least 2, self-loops cannot be distinguished anymore from data alone and multiple ground truth graphs are possible.
As a result, a maximally informative graphical representation is desirable to capture all that can be consistently inferred about the dependency graph from observational data.

\subsection{A Characterization of Markov Equivalence in E-Separation DGs}\label{sect:markov_equiv_class}

Several graphs may encode the same set of independence triples under some graphical separation criterion. When this occurs, the graphs are referred to as being \defined{Markov equivalent} and one typically aims at characterizing and finding a useful representation of the entire Markov equivalence class.
In the case of DAGs under $d$-separation, a common approach is to use a CPDAG to represent a Markov equivalence class of DAGs, despite the CPDAG itself not being a DAG \citep{pearl2009causality}.
In DGs with $\sigma$-separation, there need not exist a greatest element within an equivalence class of DGs (see \cref{app:sigma_no_maximal_element} for an example). 
This section will demonstrate that for DGs under E-separation, each equivalence class contains a greatest element, which serves as an informative representative of the class.
\begin{definition}[Markov equivalence]
Let $\cG^1 = (V, \cE_{d}^1, \cE_{bi}^1)$, $\cG^{2} = (V, \cE_{d}^{2}, \cE_{bi}^2)$ be two DMGs over a common node set $V$. We say that $\cG^1$ and  $\cG^2$ are \defined{(E-separation) Markov equivalent} if $\cI_{E}^{\cG^1} = \cI_{E}^{\cG^2}$. E-separation Markov equivalence induces an equivalence relation on the set of DMGs over $V$, and we denote the \defined{(Markov) equivalence class} of $\cG^{1}$  by \defined{$[\cG^1]_{E}$}.
\end{definition}
\begin{definition}[Maximal DMGs]
We call a DMG $\cG = (V, \cE_d, \cE_{bi})$  \defined{maximal}, if $\cG$ is complete or for all $e \in (V \times V)\setminus \cE_d$ we have $\cI_{E}^{\cG \cup \lbrace e \rbrace} \neq \cI_{E}^{\cG}$, where \defined{$\cG \cup \lbrace e \rbrace:= (V, \cE_{d} \cup \lbrace e \rbrace)$}. 
\end{definition}
\begin{definition}[Greatest element]
    We say that $\cG_0$ is the \defined{greatest element} of an equivalence class $[\cG_1]_E$ if $\cG_0 \in [\cG_1]_E$ and $\cG_2 \subseteq \cG_0$ for all $\cG_2 \in [\cG_1]_E$.
\end{definition}
To characterize Markov equivalence classes, we first show that a strongly connected component can be shielded from the outside given itself and its parents. Note that $\sccG{v} \subseteq \paG{\sccG{v}}$ if $\lvert \sccG{v} \rvert \geq 2$.
\begin{proposition}\label{prop:local_MP_Esep}
    Let $\cG = (V, \cE_d)$ be a DG and let $v \in V$. Then $(V \setminus \paG{\sccG{v}}, \sccG{v}, \paG{\sccG{v}}) \in \cI_{E}^{\cG}$.
    Moreover, we have:
    \begin{enumerate}
        \item[(i)] If $w \in V$, $w \neq v$ such that $\sccG{v} \neq \sccG{w}$, then either $(\{ v \}, \{ w \}, \paG{\sccG{w}}) \in \cI_{E}^{\cG}$ or\\ $(\{ w \}, \{ v \}, \paG{\sccG{v}}) \in \cI_{E}^{\cG}$.
        \item[(ii)] If $\{ v \} = \sccG{v}$, meaning $v$'s strongly connected component is a singleton, then 
        $(v,v) \notin \cE_d$ if and only if $(\{ v \}, \{ v \}, \paG{}) \in \cI_{E}^{\cG}$.
    \end{enumerate}    
\end{proposition}
The proof can be found in \cref{app:markov_equivalence_classes_of_dgs}.
Next, we characterize which types of edges can be added to a DG $\cG$ and their impact on the underlying independence model $\cI_{E}^{\cG}$.
\begin{lemma}\label{lemma:equivalent_edge_adding}
    Let $\cG = (V, \cE_d)$ be a directed graph, $i,j \in V$, $(i,j) \notin \cE_d$ and denote $\cG^\prime = (V, \cE_d \cup \lbrace (i,j) \rbrace)$. We then have for $i \neq j$: 
    \vspace{-2mm}
    \begin{enumerate}
        \setlength{\itemsep}{-2pt}
        \item[(i)] If $i \in \sccG{j}$, then $\cI_{E}^{\cG} = \cI_{E}^{\cG^\prime}$.
        \item[(ii)] If $i \notin \sccG{j}$ and there is a $j^\prime \in \sccG{j}$ such that $(i,j^\prime) \in \cE_d$ then $\cI_{E}^{\cG} = \cI_{E}^{\cG^\prime}$.
        \item[(iii)] If $i \notin \sccG{j}$ and there is a $j^\prime \in \sccG{j}$ such that $(i,j^\prime) \in \cE_d$ then $\cI_{E}^{\cG} \neq \cI_{E}^{\cG^\prime}$. 
    \end{enumerate}
    \vspace{-2mm}
    For $i = j$ we have:
    \vspace{-2mm}
    \begin{enumerate}
        \setlength{\itemsep}{-2pt}
        \item[(iv)] If $\lvert \sccG{j} \vert = 1$, then $\cI_{E}^{\cG} \neq \cI_{E}^{\cG^\prime}$.
    \end{enumerate}
\end{lemma}
The proof is in \Cref{app:markov_equivalence_classes_of_dgs}.
It works by showing that an open path in one graph can be used to construct an open path in the other graph.
\Cref{lemma:equivalent_edge_adding} explicitly states which edges can be added or removed from a graph to reach all other graphs within the same Markov equivalence class. It also implies that when $\cG$ is an acyclic graph with potential self-loops, its Markov equivalence class consists solely of $\cG$ itself, enabling identification of the full graph via causal discovery.

\Cref{fig:dg_equivalence_class} shows 12 DGs (i)-(xii) over 4 nodes $X^1, X^2, X^3, X^4$, all having the same set of irrelevance relations and with graph (xii) being the greatest element.
This example illustrates the characterization in \Cref{lemma:equivalent_edge_adding}.
Since $(\{ X^1\}, \{ X^1 \}, \emptyset) \in \cI_E$, none of the graphs can contain the self loop $(X^1, X^1) \in \cE_d$, however no changes to $\cI_E$ occur when adding self-loops to nodes of $\sccG{X^2} = \{ X^2, X^4 \}$.
Moreover there always exists an edge going from $\sccG{X^1}$ into $\sccG{X^2} = \{ X^2, X^4 \}$, but all options are possible. Finally, the edge $X^4 \rightarrow X^3$ can be excluded even though $X^2 \rightarrow X^4$ is always present.
With \cref{lemma:equivalent_edge_adding} we can now prove that each Markov equivalence class of DGs has a greatest element.

\begin{figure}
    \centering
    \includegraphics[width=0.8\textwidth]{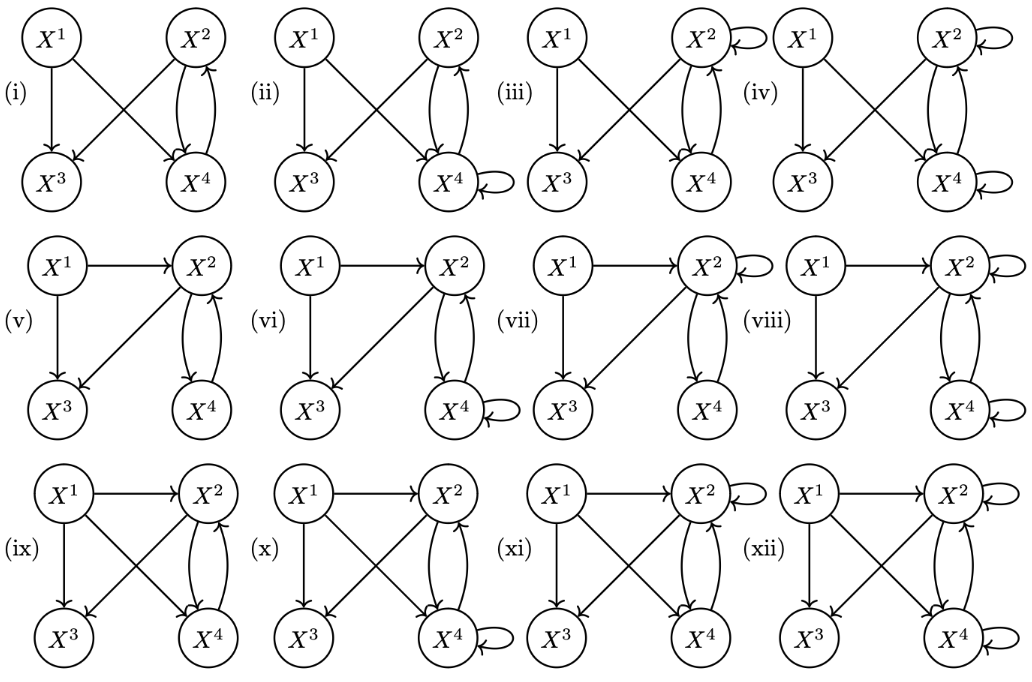}
    \vspace{-4mm}
    \caption{DGs (i)-(xii) are the elements of one Markov equivalence and (xii) is its greatest element.}\label{fig:dg_equivalence_class}
    \vspace{-3mm}
\end{figure}

\begin{theorem}[Greatest Markov equivalent DG]\label{thm:existence_max_el_DG} Let $\cG = (V, \cE_d)$ be a DG. The Markov equivalence class of $\cG$, $[\cG]_E$, contains a greatest element with respect to the partial order $\subseteq$. That is, there exists a $\hat{\cG} \in [\cG]_E$ such that $\cG^{\prime} \subseteq \hat{\cG}$ for all $\cG^{\prime} \in [\cG]_E$.

Finally, let $\cG_0, \cG_1, \ldots, \cG_m$, $\cG_{i} = (V, \cE_{i})$ be a finite sequence of DGs such that $\cG_0 = \hat{\cG}$ and $\cE_{i+1} = \cE_{i} \setminus \{ (v_{i}, w_{i}) \}$. Then for any of the following cases
\vspace{-2mm}
\begin{itemize}
    \setlength{\itemsep}{-2pt}
    \item[(i)] $v_i \in \scc_{w_i}^{\cG_{i}}$, $v_i \neq w_i$ and $\scc_{w_{i+1}}^{\cG_{i}} \in \mathbf{S}[\cG_{i+1}]$, meaning the strongly connected component remains strongly connected after removing the edge;
    \item[(ii)] $v_i \notin \scc_{w_i}^{\cG_{i}}$ and $\lvert \{ (v_i, w^\prime) \in \cE^{i}_{d} \: : \: w^\prime \in \scc_{w_i}^{\cG_{i}} \} \rvert \geq 2$;
    \item[(iii)] $v_i = w_i$ and $\lvert \scc_{w_i}^{\cG_{i}} \rvert \geq 2$    
\end{itemize}
we have that $\cI_{E}^{\cG_{i}} = \cI_{E}^{\hat{\cG}} = \cI_{E}^{\cG}$ for all $i \in [m]$.
\end{theorem}
The proof is in \cref{app:markov_equivalence_classes_of_dgs}.
DMGs have also been used as representations of so-called local independence in multivariate stochastic process using a graphical separation criterion known as $\mu$-separation. Markov equivalence classes of DMGs under $\mu$-separation, as well as so-called weak equivalence classes, also have a greatest element \citep{mogensen2020markov,mogensen2024weak}.

In analogy to the DAG case, where two DAGs are Markov equivalent if and only if they have the same skeleton and the same v-structures (see, e.g., \citet[Theorem 2.62]{lectures_lauritzen}), we can now establish the following graphical characterization result, which can be shown using \cref{lemma:equivalent_edge_adding}.
\begin{theorem}[Characterization of Markov Equivalent DGs]
Let $\cG_1 = (V, \cE_d^{1})$ and $\cG_2 = (V, \cE_d^{2})$ be DGs over $V$. Then $\cI_{E}^{\cG_1} = \cI_{E}^{\cG_2}$ is equivalent to the following three properties:
\vspace{-2mm}
\begin{enumerate}
    \setlength{\itemsep}{-2pt}
    \item[(i)] $\mathbf{S} (\cG_1 ) = \mathbf{S} (\cG_2)$, i.e., the strongly connected components coincide;
    \item[(ii)] for $[v] \in \mathbf{S} (\cG_1)$ with $\lvert \scc_{v}^{\cG_1} \rvert = 1$
    \vspace{-2mm}
    \begin{align*}
        \pa_{\scc_{v}^{\cG_1}} = \pa_{\scc_{v}^{\cG_2}}\:;
    \end{align*}
    \vspace{-4mm}
    \item[(iii)] for $[v] \in \mathbf{S} (\cG_1)$ with $\lvert \scc_{v}^{\cG_1} \rvert \geq 2$
    \vspace{-2mm}
    \begin{align*}
        \pa_{\scc_{v}^{\cG_1}} \setminus \scc_{v}^{\cG_1} = \pa_{\scc_{v}^{\cG_2}} \setminus \scc_{v}^{\cG_2}\:;
    \end{align*}
    \vspace{-4mm}
\end{enumerate}  
\end{theorem}

\subsection{E-Separation in the Presence of Latent Variables}

Latent (or unobserved) variables are typically represented in graphical frameworks by mixed edges, e.g., bi-directed edges.
The main purpose of graphs with mixed edges is to graphically represent independencies in marginalized distributions (where the latents have been marginalized out).
Accordingly, one can obtain graphs with mixed edges via a `marginalization' operation on the DG that still contains the latent variables.
For example, ADMGs are often obtained via a latent projection of a DAG and represents (conditional) independencies in the marginal joint distribution over observed variables.

In this section, we show that the global Markov property, which we have shown for DGs in \cref{prop:global_MP_emilio_granger}, with respect to a DMG obtained via latent projection from a larger DG is inherited from the DG.
To keep this paper self-contained, we briefly recap the notion of \emph{latent projection} and \emph{marginal independence model}.
\begin{definition}[The latent projection]
    Let $\cG = (V, \cE_d, \cE_{bi})$ be a DMG, $V = V_{obs} \cupdot V_{lat}$. The \defined{latent projection of $\cG$ on $V_{obs}$} :=  $\cG^\prime = (V_{obs}, \cE_{d}^\prime, \cE_{bi}^\prime )$ with 
    \begin{align*}
        \cE_{d}^\prime := \big\lbrace &(v_1, v_2) \mid v_1, v_2 \in V_{obs} \text{ and there exists a non-trivial walk} \\
        & \pi = \lbrace (w_i, e_i)\rbrace_{i \in [n]} = v_1 \rightarrow \ldots \rightarrow v_2 \text{ with } w_i \in V_{lat} \text{ for all } i \in [n] \big\rbrace\:, \\
        \cE_{bi}^\prime := \lbrace &[(v_1, v_2)] \mid v_1, v_2 \in V_{obs}  \text{ and there exists a non-trivial walk} \\
        & \pi = \big\lbrace (w_i, e_i)\rbrace_{i \in [n]} = v_1 \rcirclearrow  \ldots \lcirclearrow v_2\text{ with } w_i \in V_{lat} \text{ for all } i \in [n] \text{ such that } \coll{\pi} = \emptyset \big\rbrace\:.
    \end{align*}
\end{definition}
\begin{definition}[Marginal independence model]
Assume $\cI \subseteq \cP(V)^3$ is an abstract independence model over V. The \defined{marginal independence model of $\cI$ over $V_{obs} \subseteq V$} is defined as 
\begin{equation*}
    \restr{\cI}{V_{obs}}:= \big\lbrace (A,B,C) \in \cI \: \mid \: A,B,C \subseteq V_{obs} \big\rbrace\:.
\end{equation*}    
\end{definition}
Analogous to other graphical frameworks (e.g., \citealp[Theorem 3.12]{mogensen2020markov}) we can show that a DMG obtained by marginalization of a larger directed graph has the same independence relations among the observed (`not-marginalized-out') nodes than the original graph with respect to general E-separation.
\begin{proposition}\label{prop:for_Markov_property_partially_observed}
Let $\cG = (V, \cE_d)$ be a directed graph, $V = V_{obs} \cupdot V_{lat}$, $A,B,C \subseteq V_{obs}$ (we assume without loss of generality that $A \cap C = \emptyset$) and $\cG^\prime = (V_{obs}, \cE_{d}^\prime, \cE_{bi}^\prime )$ the latent projection of $\cG$. Then 
\begin{equation*}
    (A,B,C) \in \cI_{E}^{\cG} \: \Leftrightarrow \: (A,B,C) \in \cI_{E}^{\cG^{\prime}}\:.
\end{equation*}
\end{proposition}
The proof is in \cref{app:proof_latent_markov}. With this we can then extend the global Markov property from \cref{prop:global_MP_emilio_granger} to the partially observed setting.
\begin{proposition}[Global Markov property for Latent Models]\label{prop:global_MP_latent_models}
    Let $V \cong [d], d \in \bN$, $\lbrace X^{i}_t \rbrace_{i \in V}$ the coordinate processes of a solution of \cref{eq:SDEmodel} for $t \in [0,1]$, $\cG = (V , \cE_{d})$ the adjacency graph defined by \cref{eq:SDEmodel}, $V= V_{obs} \cupdot V_{lat}$ a partition of $V$ into observed and latent nodes and $\cG^{\prime} = (V_{obs}, \cE^{\prime}_d, \cE^{\prime}_{bi})$ the latent projection of $\cG$ on $V_{obs}$. Then for $A,B,C \subseteq V_{obs}$ it holds
    \begin{align}\label{eq:sigma_global_MP_hidden}
        A \indep_{E}^{\cG^\prime} B \given C \: \Rightarrow \: X^{A} \indepsh X^{B} \given X^{C}\:.
    \end{align}
\end{proposition}

\subsection{Properties of DMGs with respect to E-Separation}

\begin{wrapfigure}{r}{0.45\textwidth}
    \centering
    \vspace{-5mm}
    \includegraphics[width=0.40\textwidth]{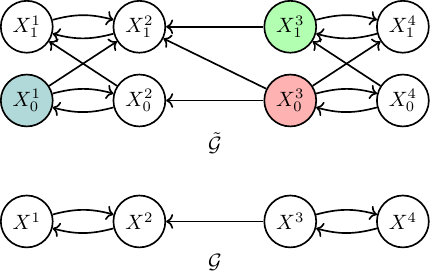}
    \vspace{-5mm}
    \caption{An inducing path connects $X^1$ and $X^3$ in $\cG$, and $X^3$ (red) E-separates $X^3$ (green) from $X^1$ (teal).}
    \label{fig:new_notion_inducing_path}
    \vspace{-9mm}
\end{wrapfigure}
In this section, we investigate the properties of Markov equivalence classes of general DMGs with respect to the graphical separation criterion $\indep_E$. We once more align our methodology with the approach outlined by \citet{mogensen2020markov} for $\mu$-separation.

The separability of pairs of nodes and their subsequent inclusion in the independence model is reflected by the presence of so-called \emph{inducing walks} or \emph{paths} in various graph classes and their graphical separation criteria (see, e.g., \citet{jiri_zhang_fci} for ancestral graphs, \citet[Chapter 11]{forre2023causality} for CDMGs 
and $\sigma$-separation, or \citet{mogensen2020markov} for DMGs and $\mu$-separation).
Since E-separation is an asymmetric version of $\sigma$-separation, we also have to define an altered, asymmetric version of inducing paths first.
\Cref{fig:new_notion_inducing_path} highlights that this definition has to incorporate that the last edge before entering a strongly connected component $\sccG{w}$ has to point into the strongly connected component. 

\begin{definition}[(Asymmetric) inducing walks] 
    Let $\cG = (V, \cE_d, \cE_{bi})$ be a DMG and $v,w \in V$. An \defined{(asymmetric) inducing path from $v$ to $w$} in $\cG$ is a non-trivial path $v=v_0 \overset{e_0}{\sim} v_1  \overset{e_1}{\sim} \ldots \overset{e_{n-1}}{\sim} v_n \overset{e_n}{\sim} v_{n+1} = w$ such that $\coll{\pi} \subseteq \anG{\lbrace v,w \rbrace }$, each $v_i \in \ncoll{\pi}\setminus \lbrace v_0 \rbrace$ is unblockable, and the last edge before entering $\sccG{w}$ has to be into the $\sccG{w}$. 

\end{definition}
We can now show the first part of the relation between separability and inducing paths. 
\begin{proposition}\label{prop:ind_path_inseparable_first_direction}
    Let $\cG = (V, \cE_d, \cE_{bi})$ be a DMG. If there exists an (asymmetric) inducing path $\nu: v=v^{0} \overset{e_0}{\sim} \ldots \overset{e_n}{\sim} w$ in $\cG$ from $v$ to $w$, then $(\lbrace v \rbrace,\lbrace w \rbrace,C) \notin \cI_{E}$ for all $C \subseteq V \setminus \lbrace v \rbrace$, meaning that $w$ is inseparable from $v$ in $\cG$. 
\end{proposition}
A direct consequence is the following `symmetry' relation of inducing paths to nodes within a strongly connected component. 
\begin{lemma}
    Let $\cG = (V, \cE_d, \cE_{bi})$ be a DMG. If there exists an (asymmetric) inducing path $\nu: v=v^{0} \overset{e_0}{\sim} \ldots \overset{e_n}{\sim} w$ in $\cG$ from $v$ to $w$ and $u \in \sccG{w}$ then there also exists an (asymmetric) inducing path from $v$ to $u$.
\end{lemma}
With a similar approach as in \cref{lemma:equivalent_edge_adding}, we can now show the following characterizations of the independence model, which states that two nodes from different strongly connected components, that are connected by a bidirected edge can not be separated. 
\begin{lemma}\label{lemma:trivial}
Let $\cG = (V, \cE_{d}, \cE_{bi})$ be a DMG and $v,w \in V$ nodes such that $w \notin \sccG{v}$. If there exists $v^\prime , w^\prime \in V$ such that $v^\prime \in \sccG{v}$, $w^\prime \in \sccG{w}$ and $[(v^\prime, w^\prime)] \in \cE_{bi}$ then
\begin{align}
    (\lbrace \hat{v} \rbrace,\lbrace \hat{w} \rbrace,C) \notin \cI_{E} \qquad \forall \hat{v}, \hat{w} \in \sccG{v} \cupdot \sccG{w}, C \subseteq V \setminus \lbrace \hat{v} \rbrace\:.
\end{align}
\end{lemma}
This follows from the existence of the path $\hat{v}_0 \leftarrow \ldots \leftarrow v^\prime_0 \leftrightarrow w^\prime_1 \rightarrow \ldots \rightarrow \hat{w}_{0}$ in $\tilde{\cG}$. Finally, in \cref{app:E_Sep_Ind_DMGs}, we prove that one can and arbitrary many bidirected edges between strongly connected components that are already connected by such a bidirected edge without altering the independence model. 
\begin{proposition}\label{prop:equivalent_bidirected_edge_adding}
    Let $\cG = (V, \cE_d, \cE_{bi})$ be a DMG, $i,j \in V$, $i\neq j$, $[(i,j)] \notin \cE_{bi}$ and denote $\cG^\prime = (V, \cE_d, \cE_{bi} \cup \lbrace [(i,j)] \rbrace )$. Then, if $i \notin \sccG{j}$ and there exist $i^\prime \in \sccG{i}, j^\prime \in \sccG{j}$ such that $[(i^\prime,j^\prime)] \in \cE_{bi}$, we have that $\cI_{E}^{\cG} = \cI_{E}^{\cG^\prime}$.
\end{proposition}

\xhdr{A Conjecture about the greatest equivalent DMG.}
We conjecture that each E-separation Markov equivalence class of directed mixed graphs (DMGs) contains a greatest element, analogous to the findings for directed graphs (DGs) in \cref{thm:existence_max_el_DG} and to $\mu$-separation DMGs \cite[Theorem 5.9]{mogensen2020graphical}. This conjecture is supported by experiments conducted on DMGs up to $d=4$, which we describe in the following section.

\section{Experiments}

For initial insights into whether each E-separation Markov equivalence class of DMGs contains a greatest element, we computationally tested whether this is true exhaustively for all DMGs up to $d=4$ nodes.
After procedurally generating all possible DGs and DMGs with $d \in \{2, 3, 4\}$ nodes, we calculate all separation triples for each one. Among all graphs with the same separation triples (i.e., within each Markov equivalence class), we then exhaustively search for a greatest element. For $d=4$, we searched over 65536 DGs and 4194304 DMGs, respectively. Due to left- and right composability of $\cI_E$ (see \cref{prop:esep_ternary}), we only need to consider single nodes (instead of sets of nodes) in the first and second argument, which helped speeding up the computations. Our hypothesis holds true for up to 4 nodes and we hypothesize that it holds for all $d$, i.e., a greatest element exists not only for DGs, but also for all DMGs. Moreover, in \cref{app:causal_discovery_algorithm}, we provide empirical evidence that the Markov equivalence class of SDE models of the form \cref{eq:SDEmodel} can be estimated from data, serving as a proof-of-concept to validate our theoretical findings.

\section{Discussion and Conclusion}
In this work, we have extended the causal discovery framework of \citet{manten2024sigker} to cyclic and partially observed SDE-models. We introduce E-separation by partitioning the time interval into past and future segments to obtain an asymmetric separation criterion that is sensitive to the direction of time. E-separation thus provides an extension of $\sigma$-separation. 
Our framework accommodates not only cyclic dependencies---arguably ubiquitous in dynamical settings---but also partial observations. This facilitates causal discovery in realistic, continuous-time models. In the fully observed setting, we characterized the Markov equivalence class explicitly and showed the existence of a greatest element---a parsimonious representation of the Markov equivalence class that is discoverable from observational data (\cref{app:causal_discovery_algorithm}). Furthermore, based on computational verification for DMGs with up to four nodes, we conjecture that E-separation Markov equivalence classes of DMGs also contain a greatest element.

\acks{This work is supported by the DAAD programme Konrad Zuse Schools of Excellence in Artificial Intelligence, sponsored by the Federal Ministry of Education and Research. This work was supported by the Helmholtz Association’s Initiative and Networking Fund on the HAICORE@FZJ partition. We thank Emilio Ferrucci and Cristopher Salvi for useful discussions throughout the project. We also thank Emilio for being the namesake of E-separation.}

\bibliography{bibliography}
\clearpage

\appendix

\section{Proofs}\label{app:proofs}

\subsection{Proofs of Asymmetric Graphoid Properties \cref{prop:esep_ternary} and \cref{prop:indepsh_ternary}}\label{app:ternary_proofs}

\begin{proof}[Proof of \cref{prop:esep_ternary}]
\begin{enumerate}
    \setlength{\itemsep}{-2pt}
    \item[(LR):] Let $A,B \subseteq V$ and $\pi = a_0 \sim \ldots \sim b^{1}$ a walk in $\tilde{\cG}$, $a \in A$, $b \in B$. As conditioning on an endpoint $\sigma$-blocks the walk, clearly $\pi$ is $\sigma$-$\left( A_0 \cup (A_1 \setminus B_1)\right)$-blocked.  
    \item[(LD):] Assume $(A,B,C) =  (A^\prime \cupdot D,B,C) \in \cI_E^\cG$ with $A^\prime := A \setminus D$, $d \in D$, $b \in B$ and let $\pi = d_0 \sim \ldots \sim b^1$ be a walk in $\tilde{\cG}$. Then $\pi$ is also a walk from $A_{0}$ to $B_1$. By assumption it is $\sigma$-$\left( C_0 \cup (C_1 \setminus B_1) \right)$-blocked. Thus $(D,B,C) \in \cI_E^\cG$. 
    \item[(LC):] We have to show that $ (A,B,C)  \in \cI_{E}^{\cG} \: \land \: (D,B,A \cup C) \in \cI_{E}^{\cG}$ implies $(A \cup D,B,C) \in \cI_{E}^{\cG}$.\\
    Assume instead $(A \cup D,B,C) \notin \cI_{E}^{\cG}$; let $\pi = a_0 \sim \ldots \sim b_{1}$, $b \in B$, $a \in A \cup D$ 
    $\sigma$-$(C_0 \cup (C_1 \setminus B_1))$-open shortest path. This means the colliders satisfy $\coll{\pi} \subseteq \an_{(C_0 \cup (C_1 \setminus B_1))}^{\tilde{\cG}}$ and for the blockable noncolliders it holds $\ncollb{\pi} \cap (C_0 \cup (C_1 \setminus B_1)) = \emptyset$. Moreover by assumption $a_0 \in A_0 \cup D_0$, $b_1 \in B_1$ the only elements of $A_0 \cup D_0$ and $B_1$ on $\pi$, respectively.
    
    \defineddeutsch{Case $a_0 \in D_0$, $a_0 \in A_0$:} Then we have a path $\pi: a_0 \sim \ldots \sim b_1$ from $D_0$ to $B_1$ such that $\coll{\pi} \subseteq \an_{C_0 \cup (C_1 \setminus B_1)}^{\tilde{\cG}} \subseteq \an_{A_0 \cup C_0 \cup (A_1 \cup C_1 \setminus B_1)}^{\tilde{\cG}}$ (meaning all colliders are still open) and we therefore only have to care about the blockable noncolliders. Note that $\ncollb{\pi} \cap (C_0 \cup (C_1 \setminus B_1)) = \emptyset$. Furthermore we can assume $\ncollb{\pi} \cap (A_0 \cup A_1 \setminus B:1) = \emptyset$. This is because by assumption $\pi \cap A_0 = \emptyset$ and if there exists an $\tilde{a}_1 \in A_1$ which is a blockable non-collider on $\pi$, meaning there exists a $k \in \bN$ such that $\overset{e_k}{\rcirclearrow} \tilde{a}_1 \overset{e_{k+1}}{\lcirclearrow}$ (with possibly other orientations). If $e_{k+1} = \lcirclearrow$ then we can immediately construct an open path from $A_0$ to $B_1$ and if $e_k$ and $e_{k+1}$ are of the form $\leftarrow$, there has to exists an $r<k$ and a collider at $v^{k}$ which is by definition open as well as all non-colliders on the subpath $v^{r}_1 \leftarrow v^{r+1}_{1} \leftarrow \ldots \leftarrow \tilde{a}_1$. But then we can simply construct the open path $a_{0} \rightarrow \ldots \rightarrow v^{r+1}_0 \rightarrow v^{r}_1 \leftarrow v^{r+1}_1 \leftarrow \ldots \leftarrow \tilde{a}_1 \sim \ldots \sim b_1$ contradicting $(A,B,C) \in \cI_{E}^{\cG}$. Thus we can assume also $\ncollb{\pi} \cap (A_0 \cup A_1 \setminus B:1) = \emptyset$.
    
    But since now all blockable noncolliders are also open, making $\pi$ $\sigma$-$(A_0 \cup C_0 \cup ((A_1 \cup C_1) \setminus B_1)$-open, we obtain a contradiction to $(D,B,A \cup C) \in \cI_{E}^{\cG}$.
    
    \defineddeutsch{Case $a_0 \in A_0$:} Then we have a path $\pi: a_0 \sim \ldots \sim b_1$ from $D_0$ to $B_1$ such that $\coll{\pi} \subseteq \an_{C_0 \cup (C_1 \setminus B_1)}^{\tilde{\cG}}$ and also $\ncollb{\pi} \cap (C_0 \cup (C_1 \setminus B_1)) = \emptyset$. We therefore obtain a contradiction to $(A,B,C) \in \cI_{E}^{\cG}$.
    \item[(LCo):] This is obvious as $E$-separation is defined in terms of $\sigma$-separation, which is defined node-wise.
    \item[(RCo):] Holds with the same justification as (LCo).
\end{enumerate}
\end{proof}
\noindent To prove \cref{prop:indepsh_ternary}, we first recall the following two concepts, taken from \citet{lectures_lauritzen}. 
\begin{definition}[Conditional independence]
Let $(\Omega, \cF, P)$ be a probability space, $\cA, \cB \subseteq \cF$ and $\cH$ a sub-$\sigma$-algebra of $\cF$. The two sets $\cA$ and $\cB$ are \defined{conditional independent} given $\sigma$-algebra $\cH$, \symbdef{$\cA \indep \cB \given \cH$} $:\Leftrightarrow$
\begin{align*}
    \cA \indep \cB \given \cH \: :\Leftrightarrow& \: A \indep B \given \cH \quad \forall A \in \cA, B \in \cB \\
    \Leftrightarrow&\:  P(A \cap B \given \cH) \overset{\text{a.s.}}{=}  P(A \given \cH)  P( B \given \cH)
\end{align*}
    where $P(A \given \cH ) = \bE [\chi_{A} \given \cH]$.
\end{definition}
An equivalent definition of the above conditional independence is given by the following theorem (\citet[Theorem 2.10]{lectures_lauritzen}): 
\begin{theorem}
    Let $(\Omega, \cF, P)$ be a probability space, $\cA, \cB \subseteq \cF$ and $\cH$ a sub-$\sigma$-algebra of $\cF$. Then: 
    \begin{align*}
        \cA \indep \cB \given \cH \: \Leftrightarrow \: P(B \given \cA \lor \cH) = P(B \given  \cH) \quad B \in \cB
    \end{align*}
\end{theorem}
\begin{proof}[Proof of \cref{prop:indepsh_ternary}]
    The stated properties are all direct consequences of conditional independence (of $\sigma$-algebras). We use the notation: 
    \begin{align*}
        \cF_{s}(X^{A}) &:= \sigma \left( \lbrace X_{t^\prime}^{a} \: : \: t^\prime \leq s , a \in A \rbrace \right) \\
        \cF_{s,h}(X^{A}) &:= \sigma \left( \lbrace X_{t^\prime}^{a} \: : \: s \leq t^\prime \leq s+h , a \in A \rbrace \right)
    \end{align*}
    and $P(A \given \cH) := \bE [\chi_{A} \given \cH]$ where $\cH$ a sub-$\sigma$-algebra. 
    \begin{enumerate}
        \item[(LR):] We have to show that $ (A,B,A) \in \cI_{s,h}$. \\
        This holds as 
        \begin{align*}
            P (\tilde{B} \mid \underbrace{\cF_s^{A} \lor \cF_{s+h}^{A}}_{=\cF_{s+h}^{A}}) = P (\tilde{B} \mid \cF_{s+h}^{A}) \qquad \forall \tilde{B} \in \cF_{s,h}^{B}
        \end{align*}
        \item[(LD):] We have to show that $(A,B,C) \in \cI_{s,h}$, $D \subseteq A$ implies $(D,B,C) \in \cI_{s,h}$ \\
        By assumption 
        \begin{align*}
            P(\tilde{A} \cap \tilde{B} \mid \cF_{s+h}^{C}) \overset{a.s.}{=} P(\tilde{A} \mid \cF_{s+h}^{C}) P(\tilde{B} \mid \cF_{s+h}^{C}) \quad \tilde{A} \in \cF_{s}^{A}, \tilde{B} \in \cF_{s}^{B}
        \end{align*}
        and since $\cF_{s}^{A} \subseteq \cF_{s}^{D}$ the statement follows.
    \end{enumerate}
\end{proof}

\subsection{Basic Properties of the Lifted Dependency Graph}\label{app:proofs_on_graphs}

\begin{proof}[Proof of \cref{lemma:lemma_walk_to_path}]
    Let $\pi : a_0 \overset{e_0}{\sim} v^1_{i_1} \overset{e_0}{\sim} \ldots \overset{e_n}{\sim} v^{n+1}_{i_{n+1}}= b_1$ be $C_0 \sqcup (C_1 \setminus \lbrace b_1 \rbrace)$-$\sigma$-open walk. and  assume that a node appears multiple, thus we take the smallest $k \in [n]$ such that 
    \begin{align*}
        \ell_k = \max \lbrace \lbrace \ell \in \dblbrackets{k+1, n+1} : v^{k}_{i_k} = v^{\ell}_{i_\ell} \rbrace > k
    \end{align*}
    If $v^k_{i_k} = b_1$ we can just use the subwalk $\tilde{\pi} : a_0 \overset{e_0}{v^{1}_{i_1}} \overset{e_1}{\sim} \ldots \overset{e_{k-1}}{\sim} v^{k}_{i_k} = b_1$. We can therefore assume $k < \ell_k < n+1$ and we consider the walk 
    \begin{align*}
        \tilde{\pi} : a_0 \overset{e_0}{\sim} v^{1}_{i_1} \overset{e_1}{\sim} \ldots \overset{e_{k-1}}{\sim} v^{k}_{i_k} \overset{e_{\ell + 1}}{\sim} v^{\ell_k +1}_{i_{\ell_l + 1}} \sim \ldots \sim b_1
    \end{align*}
    and we have to show now that it is still $C_0 \sqcup C_1 \setminus \lbrace b_1 \rbrace$-$\sigma$-open.
    
    \defineddeutsch{Case $\overset{e_{k-1}}{\lcirclearrow} v^{k}_{i_k} \overset{e_{\ell_k}}{\rcirclearrow}$:} if $\overset{e_k}{\rcirclearrow}$ or $\overset{e_{\ell_k}}{\lcirclearrow}$ then we immediately have that $v^{k}_{i_k} \in \anG{C_0 \sqcup C_1 \setminus \lbrace b_1 \rbrace}$ hence the collider and thus $\tilde{\pi}$ are open.
    
    If we instead have $\overset{e_k}{\rightarrow}$ or $\overset{e_{\ell -1}}{\leftarrow}$ then there exists 
    \begin{align*}
        \hat{k} = \max \lbrace \hat{k}^\prime \geq k : \overset{e_{\hat{k}^\prime- 1}}{\rightarrow} v^{\hat{k}^\prime}_{i^{\hat{k}^\prime}} \rcirclearrow \rbrace \leq \ell_k -1
    \end{align*}
    and again we have $\an_{C_0 \sqcup C_1 \setminus \lbrace b_1 \rbrace }^{\tilde{\cG}}$ as the collider $v^{\hat{k}}_{i_{\hat{k}}}$ needs to be open.
    
    \defineddeutsch{Case $v^{k}_{i_{k}} \in \ncollb{\tilde{\pi}}^{\tilde{\cG}}$: ($\overset{e_{k-1}}{\leftarrow} v^{k}_{i_k} \overset{e_{\ell_k}}{\lrcirclearrow}$ and $v^{k-1}_{i_{k-1}} \notin \scc_{v^{k}_{i_k}}^{\tilde{\cG}}$) or ($\overset{e_{k-1}}{\lrcirclearrow} v^{k}_{i_{k}} \overset{e_{\ell_{k}}}{\rightarrow}$ and $v^{\ell_{k}}_{i_{\ell_{k}}} \notin \scc_{v^{k}_{i_k}}^{\tilde{\cG}}$):} We consider the first case, the second follows analog. Then $v^{k}_{i_{k}}$ is already a blockable non-collider on $\pi$ with $v^{k-1}_{i_{k-1}} \notin \scc_{v^{k}_{i_k}}^{\tilde{\cG}}$ thus $v^{k}_{i_{k}} \notin C_0 \sqcup C_1 \setminus \{ b_1 \}$.
\end{proof}

\subsection{Proof of the global Markov Property \texorpdfstring{\cref{prop:global_MP_emilio_granger}}{Proposition global Markov Property}}\label{app:proofmarkov}

Before giving the proof, we first need to define the concept of \textit{acyclification} (for an overview, see e.g., \citet{forre2023causality}). For our purposes, we use the following definition from \citet{forre2017markovpropertiesgraphicalmodels}: 
\begin{definition}[\defined{Acyclification}]
    Let $\cG = (V, \cE_{d})$ be a directed graph. A DAG $\cG^\prime = (V, \cE_{d}^\prime)$ is called an \defined{acyclification of $\cG$}$:\Leftrightarrow$ 
    \begin{align}
        v \to w \in \cE_{d}^\prime \: \Leftrightarrow \: v \notin \sccG{w} \: \land \: \exists w^\prime \in \sccG{w}: \: v \rightarrow w^\prime \in \cE_d
    \end{align}
    This means that we include a directed edge from $v$ to every node $w$ in a different strongly connected component, $\sccG{w}$, if $v$ has an outgoing edge into $\sccG{w}$, and we remove all edges within each strongly connected component. 
\end{definition}

\begin{remark} 
    Note that there are other definitions of acyclifications (e.g., in \citet{forre2023causality}) which also fully connect nodes within each strongly connected component in such a way that one still obtains a DAG. Such a DAG is in general not unique. In this paper, we use the minimal version above and denote the acyclification of a DG $\cG = (V, \cE_{d})$ by \symbdef{$\cG^{acy} = (V, \cE_{d}^{acy})$}. 
\end{remark}
The reason for using acyclifications is the following relation between $\sigma$- and d-separation  \cite[Theorem 2.8.2]{forre2017markovpropertiesgraphicalmodels} (note that we only state the result for directed graphs):
\begin{theorem}\label{thm:relation_sigma_d}
    Let $\cG = (V, \cE_d)$ be a DG, and let $\cG^\prime = (V, \cE_{d}^\prime)$ be a DG with $\cE_{d}^\prime \subseteq \cE_d^{acy} \cup \cE_d^{scc}$ where $\cE_d^{scc}:= \lbrace v \to w \: \mid \: v \in V, w \in \sccG{v} \rbrace$. Then for $A,B,C \subseteq V$: 
    \begin{align*}
        A \indep_{\sigma}^\cG B \given C \: \Rightarrow  \: A \indep_{d}^{\cG^\prime} B \given C 
    \end{align*}
\end{theorem}
We now come to the proof of the global Markov property: 
\begin{proof}[Proof of \cref{prop:global_MP_emilio_granger}]
The proof is a direct application of \citet[Theorem B.2]{Forre_Constraint}, based on the concept of \textit{acyclification} (for each DMG with $\sigma$-separation, there exists an ADMG that encodes the same $d$-separation triple) and the fact that each coordinate process on the intervals $[0,s]$ (and $[s,s+h]$) is determined by the Borel-measurable functions $F_{0}^{j}$ defined as

\begin{align*}\scriptstyle{
    \begin{cases}
\bR^{\dim (\sccG{j})} \times C^{0} \left( [0,s], \bR^{\dim^{W}(\sccG{j})}\right) \times C^{0} \left( [0,s], \bR^{\dim(\paG{\sccG{j}}\setminus \sccG{j})}\right) \rightarrow C^{0} \left( [0,s], \bR^{n_j}\right) \\ 
(X_0^{\sccG{j}}, W_{[0,s]}^{\sccG{j}}, X_{[0,s]}^{\paG{\sccG{j}}\setminus \sccG{j}}) \mapsto X_{[0,s]}^{j}
    \end{cases}}
\end{align*}
respectively $F_{1}^{j}$ being defined as 
\begin{align*}\scriptstyle{
\begin{cases}
C^{0} \left( [s,s+h], \bR^{\dim^{W}(\sccG{j})}\right) \times C^{0} \left( [0,s], \bR^{\dim(\sccG{j})}\right) \times \\ C^{0} \left( [0,s+h], \bR^{\dim(\paG{\sccG{j}}\setminus \sccG{j})}\right) \rightarrow C^{0} \left( [s,s+h], \bR^{n_j}\right) \\ 
\left((W_{s+t}^{\sccG{j}} - W_{s}^{\sccG{j}})_{0\leq t\leq h},X_{[0,s]}^{\sccG{j}},  X_{[0,s+h]}^{\paG{\sccG{j}}\setminus \sccG{j}}\right) \mapsto X_{[s,s+h]}^{j}
    \end{cases}}
\end{align*}
for each $j \in \sccG{k} = \sccG{j}$ where $X_0^j$ and, respectively, $X_0^{\sccG{j}}$) are  the initial conditions, which are independent of  the Brownian path segments $W_{[0,s]}^{\sccG{j}}, W_{[s,s+h]}^{\sccG{j}}$.

Note that the above definitions for the mappings are not defined over the entire space of continuous functions on the interval and it also suffices to define them on a measurable set of paths that includes all solutions to the SDEs over the respective intervals as for instance $F^j_1$ is defined as 
\begin{align*}
& F^{j}_1((W_{s+t}^{\sccG{j}} - W_{s}^{\sccG{j}})_{0\leq t\leq h},X_{[0,s]}^{\sccG{j}},  X_{[0,s+h]}^{\paG{\sccG{j}}\setminus \sccG{j}}) = X^j_{[s,s+h]} = \text{solution of} \: X_t^{j}   \\ = & \int_{s}^{t} \mu^{j} (X_{[0,s]}^{\sccG{j}} * X_{[s,t^\prime]}^{\sccG{j}}, X_{[0,s]}^{\scriptstyle{\paG{\sccG{j}} \setminus \sccG{j}}} * X_{[s,t^\prime]}^{\scriptstyle{\paG{\sccG{j}} \setminus \sccG{j}}}) dt^\prime\\
+& \int_{s}^{t} \sigma^{j} (X_{[0,s]}^{\sccG{j}} * X_{[s,t^\prime]}^{\sccG{j}}, X_{[0,s]}^{\scriptstyle{\paG{\sccG{j}} \setminus \sccG{j}}} * X_{[s,t^\prime]}^{\scriptstyle{\paG{\sccG{j}} \setminus \sccG{j}}}) dW^{\dim^W_j}_{t^\prime} 
\end{align*}
where $*$ denotes the path concatenation of the solution on the two intervals and we have removed the dependence on the initial condition $X^{j}_s$. By \citet[Theorem 10.4]{RW00}, $F^{j}_0$ and $F^{j}_1$ are well-defined and measurable.
We can now adapt the technique used in \citet[Theorem B.2]{Forre_Constraint} to our setting. The proof works by induction on the number of strongly connected components $\lvert \mathbf{S} (\cG)\rvert  = N$. 
\begin{itemize}
        \item[$(IB)$:] Let $V = \sccG{v}$, $v \in V$. (the case $N=1$) \\
        \defineddeutsch{Case $((v,v) \in \cE_d \land \lvert \sccG{v} \rvert = 1) \lor \lvert \sccG{v} \rvert \geq 2$}: For $a,b \in V$ there exists a directed walk $a_0 \to v_1^1 \to v_1^2 \to \ldots \to b_1$ that is only $\sigma-C$-blocked if $a_0 \in C$ (`left-redundancy') hence clear the global MP holds. In addition we have the acyclification of $\tilde{\cG}$ given by $\tilde{\cG}^{acy} = (V_0 \sqcup V_1, \tilde{\cE_d}^{acy})$ DAG with $\tilde{\cE_d}^{acy} = \lbrace (i_0, j_1) \mid \forall i,j \in V \rbrace$ and most importantly, satisfying the following global MP with respect to d-separation: 
        \begin{align}\label{eq:dsep_globalMP}
            A_0 \indep_{d}^{\tilde{\cG}^{acy}} B_1 \given C_0, C_1 \backslash B_1 \: \Rightarrow \: X^{A}_{[0,s]} \indep X^B_{[s,s+h]} \given X^{C}_{[0,s]}, X^{C \setminus B}_{[s,s+h]}
        \end{align}
        for $A,B,C \subseteq V$ and by \cref{thm:relation_sigma_d}: 
        \begin{align}
            A_0 \indep_{\sigma}^{\tilde{\cG}} B_1 \given C_0, C_1 \backslash B_1  \Rightarrow A_0 \indep_{d}^{\tilde{\cG}^{acy}} B_1 \given C_0, C_1 \backslash B_1
        \end{align}
        \defineddeutsch{Case $\cG = (\lbrace v \rbrace, \emptyset)$}: This case is obvious.
        \item[$(IH)$:] Assume for a DG $\cG$ with $\lvert \mathbf{S} (\cG)\rvert  = N$, we have given an acyclification for $\tilde{\cG}$ for which the global Markov property \cref{eq:dsep_globalMP} holds. 
        \item[$(IS)$:] Assume now $\lvert \mathbf{S} (\cG)\rvert  = N+1$ and let $[v] = \sccG{v} \in \mathbf{S} (\cG)$ a terminal strongly connected component for some $v \in V$. Denoting $\cG^\prime := \restr{\cG}{V \setminus \sccG{v}}$, by (IH) we have the acyclification of the lifted dependency graph $\tilde{\cG^\prime}$, $\left( \tilde{\cG^\prime}\right)^{acy} = (V_0^\prime \sqcup V_1^\prime, (\tilde{\cE_d^\prime})^{acy})$ where $V^\prime = V \setminus \sccG{v}$.
        We assume $\lvert \sccG{v} \rvert \geq 2$, the case $\lvert \sccG{v} \rvert = 1$ with or without $(v,v) \in \cE_d$ works with a similar argument. 
        
        Then the acyclification $\tilde{\cG}^{acy}$ of $\tilde{\cG}$ is given by $\tilde{\cG}^{acy} = (V_0 \sqcup V_1, \tilde{\cE_d}^{acy})$ with edges
        \begin{align*}
            \tilde{\cE_d}^{acy} =& (\tilde{\cE_d^\prime})^{acy} \cup \lbrace (i_0,j_1) \mid i,j \in \sccG{v} \rbrace
             \\ &\cup \lbrace (i_0, j_0), (i_0, j_1), (i_1,j_1) \mid j \in \sccG{v}, i \in \paG{\sccG{v}} \setminus \sccG{v} \rbrace
        \end{align*}
        Given an arbitrary enumeration of the nodes in the strongly connected component, $\sccG{v} = \lbrace v^{1}, \ldots , v^{r} \rbrace$, we can now add each node in the order $v_{0}^{1}, \ldots, v_{0}^{r}, v_{1}^{1}, \ldots , v_{1}^{r}$ to $\tilde{\cG^\prime}$:
        So, starting from the DAG $\tilde{\cG^\prime}$, we obtain a sequence of DAGs: 
        \begin{align*}\scriptstyle
        \hat{\cG}_{1}:=\restr{\tilde{\cG}^{acy}}{V_0^\prime \sqcup V_1^\prime \cup \lbrace v_0^{1} \rbrace} ,\hat{\cG}_{2}:=\restr{\tilde{\cG}^{acy}}{V_0^\prime \sqcup V_1^\prime \cup \lbrace v_0^{1}, v_{0}^{2} \rbrace}\ldots, \hat{\cG}_{r+1}:=\restr{\tilde{\cG}^{acy}}{V_0^\prime \sqcup V_1^\prime \cup \lbrace v_0^{1}, \ldots v_0^{r}, v_{1}^{1} \rbrace}, \ldots , \hat{\cG}_{2r}:=\tilde{\cG}^{acy}.
        \end{align*}
        We can then add nodes inductively. First add the `past-nodes' $v_0^{i}$, $i \in [r]$. Since $X_{[0,s]}^{v^{i}}$ can be written as a function of $\lbrace X_{[0,s]}^{v^\prime}\rbrace_{v^\prime \in \paG{\sccG{v}}\setminus \sccG{v}}$ and the independent noises $\lbrace W_{[0,s]}^{v^\prime} \rbrace_{v^\prime \in \sccG{v}}$ (for example by Picard's successive approximation method), we can apply \citet[Lemma B.1]{Forre_Constraint} to establish 
         \begin{align}\label{eq:type_of_global_MP}
            A_0 \indep_{d}^{\hat{\cG}_{i}} B_1 \given C_0, C_1 \backslash B_1 \: \Rightarrow \: X^{A}_{[0,s]} \indep X^B_{[s,s+h]} \given X^{C}_{[0,s]}, X^{C \setminus B}_{[s,s+h]}
        \end{align}
        for $A,B,C \subseteq V$ with $B\cap \sccG{v} = \emptyset$. Adding the `future-nodes' $v_{1}^{i}$, $i \in [r]$ in similar fashion gives a global Markov property with respect to to the acyclification $\tilde{\cG}^{acy}$, 
        \begin{align*}
            A_0 \indep_{d}^{\tilde{\cG}^{acy}} B_1 \given C_0, C_1 \backslash B_1 \: \Rightarrow \: X^{A}_{[0,s]} \indep X^B_{[s,s+h]} \given X^{C}_{[0,s]}, X^{C \setminus B}_{[s,s+h]}, \quad A,B,C \subseteq V.
        \end{align*}
        \cref{thm:relation_sigma_d} then establishes \cref{eq:sigma_global_MP} and concludes the proof.
\end{itemize}
\end{proof}

\subsection{Proof of \cref{prop:global_MP_latent_models} for the latent Markov Property}\label{app:proof_latent_markov}

\begin{proof}[Proof of \cref{prop:global_MP_latent_models}]
    The proof works in both directions via contraposition.
    \defineddeutsch{`$\Leftarrow$':} Let $\pi : a_0=v^0_0 \overset{\hat{e}_0}{\sim} v^1_{\ell_1} \overset{\hat{e}_1}{\sim} \ldots \overset{\hat{e}_n}{\sim}v^{n+1}_1 = b_1$ be a $\sigma$-$C_0 \cup (C_1 \setminus B_1)$-open walk in $\tilde{\cG}$, hence $\coll{\pi} \subseteq \an_{C_0 \cup (C_1 \setminus B_1)}^{\tilde{\cG}}$ and $\ncollb{\pi} \cap (C_0 \cup (C_1 \setminus B_1)) = \emptyset$. Then for each collider $v^{k}_{\ell_k} \in \coll{\pi}$ there exists a $c^k \in (C_0 \cup (C_1 \setminus B_1))$ such that $v^{k}_{\ell_k} \in \an_{c^l}^{\tilde{\cG}}$, therefore there exists a directed path $v^{k}_{\ell_k} \rightarrow \ldots \rightarrow c^k$ of minimal length, implying that there is no other element of $(C_0 \cup (C_1 \setminus B_1))$ on it. By inserting the subwalk $v^k_{\ell_k} \rightarrow \ldots \rightarrow c^k \leftarrow \ldots \leftarrow v^{k}_{\ell_k}$ into $\pi$ for each collider $v^{k}_{\ell_k}$, we obtain a walk $\bar{\pi}: a_0=u^0_0 \overset{e_0}{\sim} u^1_{l_1} \overset{e_1}{\sim}\ldots \overset{e_m}{\sim} u^{m+1}_1 = b_1$ such that $\coll{\bar{\pi}} \subseteq (C_0 \cup (C_1 \setminus B_1))$ and $\ncollb{\pi} \cap (C_0 \cup (C_1 \setminus B_1)) = \emptyset$. Then, for each $u^{i}_{l_{i}}$ which is also in  $V_{0}^{lat} \sqcup V_1^{lat}$, there exists $k_i, \hat{k}_{i} \in \bN$ and a subwalk segment $u_{l_{i-k_i}}^{i-k_i} \sim \ldots \sim u_{l_{i+\hat{k}_i}}^{i+\hat{k}_i}$ of $\bar{\pi}$ without colliders and only the endpoints are in $V_0^{obs} \sqcup V_1^{obs}$. As there are no colliders on these segments, there exists edges $u_{l_{i-k_i}}^{i-k_i} \lrcirclearrow u_{l_{i+\hat{k}_i}}^{i+\hat{k}_i}$ in $\tilde{\cG^{\prime}}$ with the same endpoint marks. Hence by replacing the segments in question by those edges, we obtain a 
    $\sigma$-$C_0 \cup (C_1 \setminus B_1)$-open walk $\pi^\prime$ in $\tilde{\cG^\prime}$.
    \defineddeutsch{`$\Rightarrow$':} $\pi^\prime : a_0=v^0_0 \overset{\hat{e}_0}{\sim} v^1_{\ell_1} \overset{\hat{e}_1}{\sim} \ldots \overset{\hat{e}_n}{\sim}v^{n+1}_1 = b_1$ be a $\sigma$-$C_0 \cup (C_1 \setminus B_1)$-open walk in $\tilde{\cG^\prime}$. For each $e_i \notin \tilde{\cG}$ there exists an endpoint mark identical subwalk without colliders with nodes in $V_{0}^{lat} \sqcup V_1^{lat}$. By replacing these edges by those corresponding line-segments and as $C_0 \cup (C_1 \setminus B_1) \subseteq V_0^{obs} \sqcup V_1^{obs}$, we obtain a  $\sigma$-$C_0 \cup (C_1 \setminus B_1)$-open walk $\pi$ in $\tilde{\cG}$.
\end{proof}

\subsection{Proofs to establish the Markov equivalence class of DGs}\label{app:markov_equivalence_classes_of_dgs}

\begin{proof}[Proof of \cref{prop:local_MP_Esep}]
We have to show that each path (or walk) $\pi : u_0 \sim v^1_{\ell_1} \sim \ldots \sim v^{n}_{\ell_n} \sim v_1$ from an $u \notin \paG{\sccG{v}}$ is $\sigma$-$C_0 \cup C_1 \setminus \{ v_1 \}$-blocked where $\ell_k \in \{ 0,1 \}$ for each $ k \in [n]$ and $C:= \paG{\sccG{v}}$. We denote by $r := \max \{ k \in [n] \: : \: v^{k} \notin \sccG{v} \}$
the last index before entering the strongly connected component of $v$. We therefore have to consider the following three cases:\\
\defineddeutsch{Case `The walk enters $\sccG{v}$ through its parents:'} Then we have $v^r \in \paG{\sccG{v}} \setminus \sccG{j}$, $v^{r+1} \in \sccG{j}$ and $v^{r}_{\ell_r} \in \ncollb{\pi}^{\tilde{\cG}}$ a blockable non-collider on the walk $\pi$ which is conditioned upon. Hence $\pi$ is $\sigma$ $(C_0 \cup (C_1 \setminus \{ v_1 \}))$ blocked.\\
\defineddeutsch{Case `The walk enters through the children of $\sccG{v}$ but its past:'} This means we have $v^{r+1} \in \sccG{v}$, $\ell_{r+1}= 0$ such that $v_{0}^{r+1} \rightarrow v^{r}_{\ell_r}$. Then again $v^{r+1}_{0} \in \ncollb{\pi}^{\tilde{\cG}}$ a blockable collider on the walk $\pi$ which is conditioned upon. Hence $\pi$ is $\sigma$ $(C_0 \cup (C_1 \setminus \{ v_1 \}))$ blocked.\\
\defineddeutsch{Case `The walk enters through the children of $\sccG{v}$ but its future:'} This means $v^{r+1} \in \sccG{v}$, $v^{r} \in \chG{\sccG{j}} \setminus \sccG{j}$ and $v_1^{r+1} \rightarrow v_1^{r}$ the taken edge. Denoting $s = \max \{ k \in [n] \: : \: \ell_{k} \neq 1 \}$ the largest index on the path before being entirely in the future node components, we can assume $s < r$. (Otherwise the path $\pi$ would already be blocked since we condition on the past $\scc_{v_0}^{\tilde{\cG}})$.
In the case that $v^{s} = \deG{v}$, then there is a collider on the walk because $\pi$ contains edges $v_{0}^{s} \rightarrow v_{1}^{s+1}$ and $v_{1}^{r} \leftarrow v_{1}^{r+1}$ and this collider is not opened as we are in the descendants of the conditioning set; thus $\pi$ is blocked. 
If $v^{s} \notin \deG{v}$, then there $\exists s^\prime > s$ such that $v^{s^\prime} \notin \deG{v}$ but $v^{s^\prime + 1} \in \deG{v}$, hence we have the edge $v^{s^\prime} \rightarrow v^{s^\prime +1}$ and again, since $v_{1}^{r} \leftarrow v_{1}^{r+1}$, $\pi$ contains collider that is not open as it is in the descendants of the conditioning set; hence $\pi$ is blocked. This concludes the first statement.

The other two statement are immediate consequences of it:
\begin{enumerate}
    \item[(i)] Assume $v,w \in V$, $v \neq w$ and $\sccG{v} \neq \sccG{w}$ (meaning they are disjoint as it is an equivalence relation). If $v \notin \paG{\scc{w}} \setminus \sccG{w}$, then the statement follows from the preceding statement. If however $v \in \paG{\scc{w}} \setminus \sccG{w}$ in the parents of the strongly connected component of $w$, the $w$ can be E-separated from $v$ by $\paG{\sccG{w}}$ as otherwise $w$ and $v$ would be in the same strongly connected component.
    \item[(ii)] This works with a similar argument then the first since each path $v_0 \sim \ldots \sim v_1$ can be blocked by the parental set (which does not contain $v$ itself). 
\end{enumerate}
\end{proof}

\begin{proof}[Proof of \cref{lemma:equivalent_edge_adding}]
\vspace{-3mm}
    \begin{enumerate}
        \setlength{\itemsep}{-2pt}
        \item[(i)] Let $i \in \sccG{j} \: \land \: i \neq j$. Then obviously $\lvert \sccG{j} \rvert \geq 2$. \\ 
        Assume the contradiction, meaning $\exists $ $A,B, C \subseteq V$: $(A,B,C) \in \cI_{E}^{\cG}$ but $(A,B,C) \notin \cI_{E}^{\cG^\prime}$. \\ 
        Then $\exists$ $\sigma- (C_0 \cup C_1 \setminus B_1)$-open walk $\pi^\prime = a_0 \sim v^1_{k_1} \sim \ldots \sim v^{n}_{k_n} \sim b^{1}$ in $\tilde{\cG^\prime}$ for $a \in A$ and $b \in B$. If $(i_0, j_1), (i_0, j_0), (i_1, j_1) \in \pi^\prime$, then we can replace each of these edges by the respective corresponding, endpoint-identical, directed path segments $i_0 \rightarrow j^\prime_1 \rightarrow \ldots \rightarrow j_1$ (with $j^\prime_1 \rightarrow \ldots \rightarrow j_1$ in $\scc_{j_1}^{\tilde{\cG}}$), $i_0 \rightarrow  \ldots \rightarrow j_0 $ (which is inside $\scc^{\tilde{\cG}}_{i_0}$), $i_1 \rightarrow  \ldots \rightarrow j_1$ (inside $\scc^{\tilde{\cG}}_{i_1}$) and only the first path-segment $ i_0 \rightarrow j^\prime_1 \rightarrow \ldots \rightarrow j_1$ has a blockable non-collider namely $i_0$. But $i_0 \notin C_0$ as otherwise $\pi^\prime$ would already be blocked.
        
        Hence we obtain a walk $\pi : a_0 \sim w^{1}_{\ell_1} \sim \ldots \sim w_{\ell_m}^{m} \sim b_1^{1}$ in $\tilde{\cG}$ by replacing the edges $(i_0, j_1), (i_0, j_0), (i_0, j_1)$ on $\pi^\prime$ with their above mentioned edge-identical counterparts. As shown above, $\pi$ has the same set of blockable non-colliders as we only added the above mentioned segments and also the same collider nodes as the added path-segments are directed and direction-preserving, thus giving $\coll{\pi}^{\tilde{\cG}} = \coll{\pi^\prime}^{\tilde{\cG^\prime}}$. However as $\pi^\prime$ is open, each collider $c_k$ with $\an_{c_k}^{\tilde{\cG^\prime}} \cap (C_0 \cup C_1 \setminus B_1)$ and thus we can also establish the same ancestral relation in $\tilde{\cG}$ by replacing edges $(i_0, j_1), (i_0, j_0), (i_1, j_1)$ by directed path-segments $i_0 \rightarrow j^\prime_1 \rightarrow \ldots \rightarrow j_1$, $i_0 \rightarrow  \ldots \rightarrow j_0$, $i_1 \rightarrow  \ldots \rightarrow j_1$. Hence $\pi$ is $\sigma- (C_0 \cup C_1 \setminus B_1)$-open walk in $\tilde{\cG}$ as well \contradiction
        \item[(ii)] Let $i \notin \sccG{j}$ and assume there exists a $j^\prime \in \sccG{j}$ such that $ (i,j^\prime) \in \cE_d$. Assume first $\exists $ $A,B, C \subseteq V$ such that $(A,B,C) \in \cI_{E}^{\cG}$ but $(A,B,C) \notin \cI_{E}^{\cG^\prime}$.
        Then there exists a shortest $\sigma$-$(C_0 \cup C_1 \setminus B_1)$-open walk $\pi^\prime = a_0 \sim v^1_{k_1} \sim \ldots \sim v^{n}_{k_n} \sim b^{1}$ in $\tilde{\cG^\prime}$ from $a \in A$ to $b \in B$. If $(i_0, j_1), (i_0, j_0), (i_1, j_1) \in \pi^\prime$ we can again replace those by the edge-point identical, directed path-segments $i_0 \rightarrow j^\prime_1 \rightarrow \ldots \rightarrow j_1$ (with $j^\prime_1 \rightarrow \ldots \rightarrow j_1$ in $\scc_{j_1}^{\tilde{\cG}}$), $i_0 \rightarrow j^\prime_0 \rightarrow \ldots \rightarrow j_0$ (with $j^\prime_0 \rightarrow \ldots \rightarrow j_0$ in $\scc_{j_0}^{\tilde{\cG}}$), $i_1 \rightarrow j^\prime_1 \ldots \rightarrow j_1$ (with $j^\prime_1 \ldots \rightarrow j_1$ in $\scc_{j_1}^{\tilde{\cG}}$) with the only blockable non-collider being $i_0$ and $i_1$, respectively. But $i_0, i_1 \notin C_0 \cup (C_1 \setminus B_1)$ (as otherwise $\pi^\prime$ already not open in $\tilde{\cG^\prime}$) and also $i_1 \notin B_1$ as one would otherwise have a shorter $\sigma- (C_0 \cup C_1 \setminus B_1)$-open walk from $A_0$ to $B_1$. Hence we obtain a walk $\pi : a_0 \sim w^{1}_{\ell_1} \sim \ldots \sim w_{\ell_m}^{m} \sim b_1^{1}$ with the same set of blockable non-colliders as we only added the above mentioned segments and keep the same collider nodes. Since $\coll{\pi}^{\tilde{\cG}} = \coll{\pi^\prime}^{\tilde{\cG^\prime}}$ as for each collider $c_k$ with $\an{c_k}^{\tilde{\cG^\prime}} \cap (C_0 \cup C_1 \setminus B_1)$ we can also establish the same ancestral relation in $\tilde{\cG}$ by replacing edges $(i_0, j_1), (i_0, j_0), (i_1, j_1)$ by the above, directed path-segments $i_0 \rightarrow j^\prime_1 \rightarrow \ldots \rightarrow j_1$, $i_0 \rightarrow  \ldots \rightarrow j_0$, $i_1 \rightarrow  \ldots \rightarrow j_1$. Hence $\pi$ is $\sigma- (C_0 \cup C_1 \setminus B_1)$-open walk in $\tilde{\cG}$ \contradiction
        \item[(iii)] This follows immediately from \cref{prop:local_MP_Esep}.
    \end{enumerate}
\end{proof}

\begin{proof}[Proof of \cref{thm:existence_max_el_DG}]
    Let $\cG = (V, \cE_d)$ and $\cG^\prime = (V, \cE_d^\prime)$ be two DGs such that $\cI_{E}^{\cG} = \cI_{E}^{\cG^\prime}$.
    Then first of all, the set of strongly connected components in both graphs have to be the same. If otherwise there are $v \neq w \in V$ and $w \in \sccG{v}$ but $w \notin \scc^{\cG^\prime}_{v}$, then the two nodes are inseparable in $\cG$ but not in $\cG^\prime$ (either $(\lbrace v \rbrace, \lbrace w \rbrace , \pa^{\cG^\prime}_{\scc^{\cG^\prime}_{w}}) \in \cI_{E}^{\cG^\prime}$ or $(\lbrace w \rbrace, \lbrace v \rbrace , \pa^{\cG^\prime}_{\scc^{\cG^\prime}_{v}}) \in \cI_{E}^{\cG^\prime}$). We can therefore use the notion $\sccG{v}$ and $\scc^{\cG^\prime}_{v}$ interchangeably.
    Moreover, with the same argumentation, for each singleton $\lbrace v \rbrace \in \S_\cG = \S_{\cG^\prime}$ the statement  `$(v,v) \in \cE_d$ $\Leftrightarrow$ $(v,v) \in \cE^\prime_d$' has to hold.
    In addition, the DAGs of strongly connected components for both graphs have to coincide. Otherwise, if this would not be the case, there are $v,w \in V$: $w \notin \sccG{v}$ and there exists $v^\prime \in \sccG{v}$, $w^\prime \in \sccG{w}$ such that $(v^\prime, w^\prime) \in \cE_d$ but for no pair of nodes $v^{\prime \prime} \in \sccG{v}$, $w^{\prime \prime} \in \sccG{w}$ we have $(v^{\prime \prime}, w^{\prime \prime} ) \in \cE_{d}^{\prime}$. But the this means that $(\lbrace v \rbrace, \lbrace w \rbrace, \pa^{\cG^\prime}_{\scc^{\cG^\prime}_{w}}) \in \cI_{E}^{\cG^\prime}$ but not in $\cI_{E}^{\cG}$ as $\pa^{\cG^\prime}_{\scc^{\cG^\prime}_{w}}$ does not intersect $\sccG{v}$.
    Finally, if for $v \neq w$ there is the edge $[v] \rightarrow [w]$ in $\S_{\cG}$ (hence also in $\S_{\cG^\prime}$), then the nodes, from which edges are outgoing from $\sccG{v}$ to $\sccG{w}$ have to coincide. If this would not be the case, then without loss of generality there exists a $v^\prime \in (\pa^{\cG^\prime}_{\sccG{w}} \cap \sccG{v})$ such that $v^\prime \notin (\pa^{\cG^\prime}_{\sccG{w}} \cap \sccG{v})$. But then $(\lbrace v^\prime \rbrace, \lbrace w \rbrace, \paG{\sccG{w}}) \in \cI_{E}^{\cG}$ but not in $\cI_{E}^{\cG^\prime}$. This proves the above statement.
    Together with \cref{lemma:equivalent_edge_adding}, the existence of a greatest element can be concluded. The remaining statements are trivial. 
\end{proof}

\subsection{Proofs about Inducing Paths and the Markov equivalence class of DMGs}\label{app:markov_equivalence_classes_of_dmgs}

\begin{proof}[Proof of \cref{prop:ind_path_inseparable_first_direction}]
In the case that $v \in \sccG{w}$ or $v^1 \sccG{w}$ (meaning there exists $\bar{\nu}: v_0 \rightarrow w_1$ respectively $\bar{\nu}: v_0 \lcirclearrow v^{1}_{1} \rightarrow \ldots \rightarrow w^1$) the statement is obvious. We can also assume that there have to be at least three scc's involved.

We furthermore can assume that $\nu$ has minimal length. Let now $C \subseteq V \setminus \lbrace v \rbrace$, we want to show that $(\lbrace v \rbrace, \lbrace w \rbrace, C) \notin \cI_{E}$ meaning there exists a $\sigma$-$C_0 \cup (C_1 \setminus \lbrace w_1 \rbrace)$-open path.

Therefore consider the following set: 
\begin{align*}
    S_C = \big\lbrace \pi : v_0 \sim \ldots \sim w_1 \: \text{walk in} \: \tilde{\cG} : \coll{\pi} \subseteq \an^{\tilde{\cG}}_{\lbrace v_0, w_1 \rbrace \cup C_0 \cup (C_1 \setminus \lbrace w_1 \rbrace)} \land \\
    \ncollb{\pi} \cap (C_0 \cup (C_1 \setminus \lbrace w_1 \rbrace) \big\rbrace
\end{align*}
Note the relation $k \in \anG{v}$ implies $k_0 \in \an^{\tilde{\cG}}_{v_0}$ and $k_1 \in \an^{\tilde{\cG}}_{v_1}$ and also $k_0 \in \an^{\tilde{\cG}}_{v_1}$.\\
The set is not empty as from the inducing path $\nu: v=v^{0} \overset{e_0}{\sim} \ldots \overset{e_n}{\sim} w$, with $\hat{k} = \max \lbrace k \in [n]\: : \: $ we can construct the following path $\hat{\nu}: v_0 \overset{\hat{e_0}}{\sim} v_{\ell_1}^1 \overset{\hat{e_1}}{\sim} \ldots v^{\hat{k}}_{\ell{\hat{k}}} \overset{\hat{e_{\hat{k}}}}{\lcirclearrow}v_{1}^{\hat{k}+1} \sim \ldots \sim w^1$ with the following properties: 
\begin{itemize}
    \item for each collider $v^{k}_{\ell_k} \in \coll{\hat{\nu}}$ in holds $v^{k}_{\ell_k} \in \an_{\lbrace v_0, w_1\rbrace}^{\tilde{\cG}}$ since by assumption and the remark above, $v^{k} \in \anG{\lbrace v,w \rbrace}\setminus \lbrace v,w \rbrace$. 
    \item each non-collider $v^{k}_{\ell_k} \in \ncoll{\hat{\nu}}$ is unblockable
\end{itemize}
Consider now the path $\hat{\omega} \in S_C$ with a minimal number of colliders, and we denote him $v_0 = u^0_{0} \sim \ldots \sim u^{k}_{\ell_k} \sim \ldots \sim u^{m+1}_{\ell_{m+1}} = w^1$.

We want to show that $\coll{\hat{\omega}} \subseteq \an_{C_0 \cup C_1 \setminus \lbrace w_1 \rbrace}^{\tilde{\cG}}$, meaning that in this case we have found a $\sigma$-$C_0 \cup (C_1 \setminus \lbrace w_1 \rbrace)$-open path.

By definition, all colliders are in $\an^{\tilde{\cG}}_{\lbrace v_0, w_1 \rbrace \cup C_0 \cup (C_1 \setminus \lbrace w_1 \rbrace)}$, therefore have to following case-distinction:
If there would exist a collider $u_0^{k} \in \an^{\tilde{\cG}}_{v_0} \setminus \an^{\tilde{\cG}}_{C_0 \cup C_1 \setminus \lbrace w_1 \rbrace}$ (we assume $k$ to be maximal with this property) then this means there exists a directed path $v_0^{k} \rightarrow \ldots \rightarrow v_0$ that does not intersect $C_0$ and we can construct a path $v_0 \leftarrow \ldots \leftarrow \underbrace{u_0^{k} \sim \ldots \sim w^1}_{\text{rest of } \hat{\omega}}$. 
\end{proof}

\subsection{On the E-Separation Independence Model of DMGs}\label{app:E_Sep_Ind_DMGs}

\begin{proof}[Proof of \cref{prop:equivalent_bidirected_edge_adding}]
    Let $i \notin \sccG{j}$, assume there exists an $\exists i^\prime \in \sccG{i}$ and a $j^\prime \in \sccG{j}$ such that $i^\prime \leftrightarrow j^\prime \in \cE_bi$ but $i \leftrightarrow j \notin \cE_{bi}$. Assume first $\exists $ $A,B, C \subseteq V$ such that $(A,B,C) \in \cI_{E}^{\cG}$ but $(A,B,C) \notin \cI_{E}^{\cG^\prime}$.
        
        Then there exists a shortest $\sigma$-$(C_0 \cup C_1 \setminus B_1)$-open walk $\pi^\prime = a_0 \sim v^1_{k_1} \sim \ldots \sim v^{n}_{k_n} \sim b^{1}$ in $\tilde{\cG^\prime}$ from $a \in A$ to $b \in B$ and since bidirected edges do not change ancestral relations, one of the edges $i_0 \leftrightarrow j_1$, $i_0 \leftrightarrow j_0$, $i_1 \leftrightarrow j_1$ or $i_1 \leftrightarrow j_0$ is present on this path. These however can be replaced by their end-point and direction-preserving counterparts:
        \begin{align*}
            i_0 \leftarrow \ldots \leftarrow i^\prime_0 \leftrightarrow j^\prime_1 \rightarrow \ldots \rightarrow j_1 \\
            i_0 \leftarrow \ldots \leftarrow i^\prime_0 \leftrightarrow j^\prime_0 \rightarrow \ldots \rightarrow j_0 \\
            i_1 \leftarrow \ldots \leftarrow i^\prime_1 \leftrightarrow j^\prime_1 \rightarrow \ldots \rightarrow j_1 \\
            i_1 \leftarrow \ldots \leftarrow i^\prime_1 \leftrightarrow j^\prime_0 \rightarrow \ldots \rightarrow j_0 \\
        \end{align*}
        and all nodes within are unblockable non-colliders such that we can construct a $\sigma$-$(C_0 \cup C_1 \setminus B_1)$-open walk $\pi: a_0 \sim \ldots \sim b_1$, which is a contradiction to the above assumption  \contradiction
\end{proof}

\section{Causal Discovery in the fully observed SDE Model}\label{app:causal_discovery_algorithm}

In this section, we empirically show, that when applying the causal discovery algorithm introduced in \citet{manten2024sigker}, to SDEs with cyclic adjacencies, we can reliably estimate the Markov equivalence class directly from data, demonstrating the real world applicability as well as practically verifying the theoretical findings in \cref{sect:markov_equiv_class}. Since the test and algorithm were developed and thoroughly analyzed in another paper, we omit large-scale experiments and comparisons with other methods here, focusing instead on a proof-of-concept demonstration.

\subsection{The Algorithm}

\Cref{algo:ctPC} is only slightly altered in that we now use the right-decomposable notion of E-separation.

\begin{algorithm}[H]
\caption{Causal discovery for SDEs.}\label{algo:ctPC}
    \footnotesize
    \begin{algorithmic}[1]
        \State \textbf{Input:} DG $\cG = (V \cong [d], \cE_{d})$, CIT $\indep^+_{s,h}$
        \State $\tilde{V} \gets \{k_0,k_1 \mid k \in V\}$,
        $\tilde{\cE}_d \gets \{i_0 \to j_0, \ i_1 \rightarrow j_1, i_0 \rightarrow j_1 \mid (i,j) \in \cE_d \}$
        \For{$c = 0, \ldots, d-1$}
        \For{$(i,j) \in V$}
        \For{$K \subseteq V \setminus \{i\}$, $|K| = c$, such that\ $(k_0 \to j_1) \in \tilde \cE_d$ for $k \in K$}
        \If{$X^i \indep^+_{s,h} X^j \mid X^K$}
        \State $\tilde \cE_d \gets \tilde \cE_d \setminus \{i_0 \to j_0, i_1 \to j_1, i_0 \to j_1 \}$
        \EndIf
        \EndFor
        \EndFor
        \EndFor
        \State $\cG = (V,\cE_d) \gets \mathrm{collapse}(\tilde V, \tilde \cE_d)$
        \State \Return $\cG$
    \end{algorithmic}
\end{algorithm}

\subsection{Recovering the Markov Equivalence Class from Data}

To demonstrate the practical efficacy of the \cref{algo:ctPC} and verify the theoretical results presented in \cref{sect:markov_equiv_class}, we apply it by drawing 50 sets of parameters for each of the two 3-dimensional linear SDEs:   
\begin{align}\label{eq:SDE1}
    \dif \begin{pmatrix} X^1_t \\ X^2_t \\ X^3_t \end{pmatrix} = \left(\begin{pmatrix} a_{11} & 0 & 0 \\ a_{21} & a_{22} & a_{23} \\ 0 & a_{32} & a_{33} \end{pmatrix} \begin{pmatrix} X^1_t \\ X^2_t \\ X^3_t \end{pmatrix} + \begin{pmatrix} c_1 \\ c_2 \\ c_3 \end{pmatrix} \right) \dif t + \begin{pmatrix} d_{1} & 0 & 0 \\ 0 & d_{2} & 0 \\ 0 & 0 & d_{3} \end{pmatrix}  \dif \begin{pmatrix} W^1_t \\ W^2_t \\ W^3_t \end{pmatrix}.
\end{align}
and 
\begin{align}\label{eq:SDE2}
\dif \begin{pmatrix} X^1_t \\ X^2_t \\ X^3_t \end{pmatrix} = \left(\begin{pmatrix} a_{11} & a_{12} & 0 \\ 0 & a_{22} & a_{23} \\ 0 & a_{32} & a_{33} \end{pmatrix} \begin{pmatrix} X^1_t \\ X^2_t \\ X^3_t \end{pmatrix} + \begin{pmatrix} c_1 \\ c_2 \\ c_3 \end{pmatrix} \right) \dif t + \begin{pmatrix} d_{1} & 0 & 0 \\ 0 & d_{2} & 0 \\ 0 & 0 & d_{3} \end{pmatrix}  \dif \begin{pmatrix} W^1_t \\ W^2_t \\ W^3_t \end{pmatrix}.
\end{align}
with different adjacency-structures and parameters drawn as $a_{ij} \sim \cU((-1.5,1] \cup [1, 1.5))$ for $i \neq j$,  $a_{ii} \sim \cU([-0.5, 0.5])$, $c_i \sim \cU([0, 0.1))$, $d_i \sim \cU([0.3, 0.5))$. From each SDE, we draw 400 sample-paths and try to discovery the underlying causal relationships using algorithm \cref{algo:ctPC}. The result, shown in \cref{fig:toy_example_3d}, display the predicted probabilities next to each edge, confirming the theoretical expectations: $\cG_1$ closely corresponds to the equivalence class for the adjacency structure (considered as a DG) of \cref{eq:SDE1}, while $\cG_2$ corresponds to that in \cref{eq:SDE2}.  

\begin{figure}[H]
\centering
\includegraphics[width=0.7\textwidth]{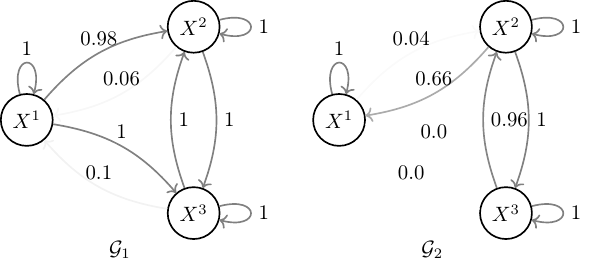}
\caption{Results from applying the algorithm to two 3-dimensional linear SDEs, showing predicted probabilities for each edge and confirming the expected equivalence classes.}
\label{fig:toy_example_3d}
\end{figure}

\section{Counterexample for greatest element for \texorpdfstring{$\sigma$}{sigma}-Separation}\label{app:sigma_no_maximal_element}

As can be seen in \cref{fig:sigma_no_max}, the graphs $\cG_1$-$\cG_5$ are all in the same equivalence class with respect to their induced graphoid $\cI_{\sigma}$, in which adjacent nodes $X^1$ and $X^2$ and $X^2$ and $X^3$ cannot be separated from each other except by conditioning on the respective endpoints and $X^1$ and $X^3$ can be separated from each other by conditioning on $X^2$. However their supremum with respect to the partial subset ordering on $\cE_d$, $\cG_6$ does not allow for a separation between $X^1$ and $X^3$ at all.

\begin{figure}[H]
\centering
\includegraphics[width=1\textwidth]{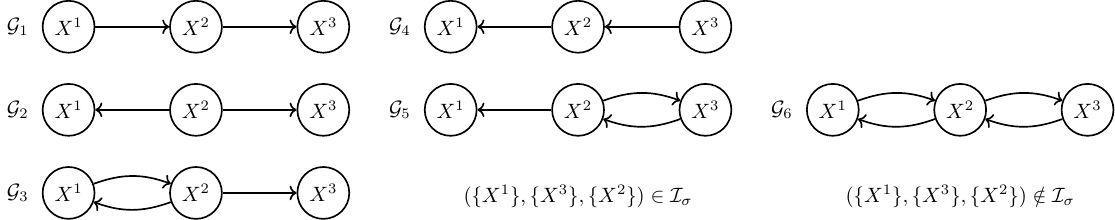}
\caption{The graphs $\cG_1$-$\cG_5$ all have the same independence model, however in their supremum with respect to partial subset ordering, $\cG_6$, $X^1$ and $X^3$ cannot be separated by $X^2$.}
\label{fig:sigma_no_max}
\end{figure}

\end{document}